\documentclass{article}

\usepackage[margin=1in]{geometry}
\usepackage{amsthm}      
\usepackage{amssymb}     
\usepackage{amsmath}     
\usepackage{algorithm}
\usepackage{algorithmicx}
\usepackage{algpseudocode}
\usepackage{xcolor} 
\usepackage{graphicx}
\usepackage{subcaption}
\usepackage{mathtools}
\mathtoolsset{showonlyrefs}
\newtheorem{theorem}{Theorem}[section]
\newtheorem{lemma}[theorem]{Lemma}
\newtheorem{definition}[theorem]{Definition}
\newtheorem{assumption}[theorem]{Assumption}

\newcounter{subroutine}
\makeatletter
\newenvironment{subroutine}[1][htb]{
    \let\c@algorithm\c@subroutine
    \renewcommand{\ALG@name}{Subroutine}
    \begin{algorithm}[#1]
        }{\end{algorithm}
}
\makeatother

\algdef{SE}[ALG]{Alg}{EndAlg}
   [2]{\textbf{Algorithm}\ \textproc{#1}\ifthenelse{\equal{#2}{}}{}{(#2)}}
   {\algorithmicend\ \textbf{Algorithm}}

\newcommand{\citep}{\cite}
\newcommand{\sigmaU}{\sigma_*(U;k^{(1)})}

\algnewcommand{\LineComment}[1]{\State \(\triangleright\) #1}

\newcommand{\ScaledPCA}{\textsc{ScaledPCA} }
\newcommand{\ScaledPCAEstimate}{\textsc{ScaledPCA}}
\newcommand{\ColSpaceEst}{\textsc{ColumnSpaceEstimate}}
\newcommand{\sv}{\mathrm{sv}}
\newcommand{\argmin}[1]{\underset{#1}{\mathrm{argmin}}}
\newcommand{\argmax}[1]{\underset{#1}{\mathrm{argmax}}}

\newcommand{\diag}{\mathrm{diag}}
\newcommand{\prev}{\mathrm{prev}}
\newcommand{\X}{\mathring{X}}
\newcommand{\hatW}{\hat W}

\newcommand{\MedianLS}{\textsc{MedianLS} }
\newcommand{\Sample}{\textsc{Sample} }

\newcommand{\scaled}{\mathrm{scaled}}

\newcommand{\Unif}{\mathrm{Unif}}

\newcommand{\inc}{\mathrm{inc}}

\newcommand{\init}{\mathrm{init}}
\newcommand{\iter}{\mathrm{iter}}
\newcommand{\prob}{\mathrm{prob}}

\newcommand{\SmoothQR}{\textsc{SmoothQR}}

\newcommand{\QR}{\mathrm{QR}}

\newcommand{\E}{\mathbb{E}}

\newcommand{\one}{\mathbf{1}}

\newcommand{\p}{_{\perp}}
\newcommand{\C}{\mathcal{C}}
\newcommand{\R}{\mathbb{R}}
\newcommand{\N}{\mathcal{N}}
\newcommand{\Nat}{\mathbb{N}}
\renewcommand{\P}{\mathcal{P}}
\newcommand{\Q}{\mathcal{Q}}

\newcommand{\med}{\mathrm{med}}

\begin{document}

\title{Recommendation on a Budget: Column Space Recovery from Partially Observed Entries with Random or Active Sampling}
\author{
  Carolyn Kim\footnote{
    Department of Computer Science, Stanford University, ckim@cs.stanford.edu}
  \and
  Mohsen Bayati\footnote{
    Graduate School of Business, Stanford University, bayati@stanford.edu}
}

\maketitle

\begin{abstract}
  We analyze alternating minimization for column space recovery of a partially observed, approximately low rank matrix with a growing number of columns and a fixed budget of observations per column.
  In this work, we prove that if the budget is greater than the rank of the matrix, column space recovery succeeds -- as the number of columns grows, the estimate from alternating minimization converges to the true column space with probability tending to one. From our proof techniques, we naturally formulate an active sampling strategy for choosing entries of a column that is theoretically and empirically (on synthetic and real data) better than the commonly studied uniformly  random sampling strategy.
\end{abstract}

\section{Introduction}
\label{section:intro}
In many applications of recommendation systems, we have data in the form of an incomplete matrix, where one dimension is growing and the other dimension is fixed.
For instance, in recommendation systems, there is a fixed set of potential products (rows of a matrix) to offer customers that arrive over time (columns of a matrix). Three other applications are choosing machine learning models (rows) for each new customer's dataset (columns) \citep{automl}, choosing which survey questions (rows) to ask to respondents (columns) that arrive sequentially \citep{survey},  or choosing which lab tests (rows) to order for each new patient (columns) \citep{huck2014utilization}.
In  these cases, there is an inherent asymmetry with respect to the dimensions in the budget: we have a budget over each column, not over each row. We could choose any machine learning model and recommend it for each dataset, or choose any survey question and give it to every user, but it is very hard to run every machine learning pipeline on an arbitrary dataset, or to give every survey question to an arbitrary respondent (indeed, in \cite{survey}, users omitting too many answers was the precise motivation for their problem). Similarly, running all lab tests on one patient siginificantly exceeds the time and cost budget per patient.

In these applications, we are often interested in approximately recovering the column space of a matrix, or equivalently, the subspace spanned by the top principal components of a data matrix. This subspace would give insights as to which machine learning models tend to perform better, which questions are most informative to ask in a survey, or which lab tests would be most valuable to order.

In particular, for a matrix that has approximately low rank $r$, we are interested in the case where we have a fixed number $k$ of entries that are sampled for each new column. We can then pose the following questions -- is it possible to recover the column space, with growing accuracy and higher confidence as $t$ increases? And if we learn the column space more accurately, does this lead to better imputation of the matrix?

In this work, we show that for an approximately rank $r$ matrix with $N$ rows and $t$ columns, when we have a budget of $k > r$ observations per column, we can recover the column space with probability tending to one (as $t$ grows) using alternating minimization when samples are randomly selected. Moreover, we establish theoretically and experimentally that an active learning strategy can help learn this subspace faster. We also show experimentally that more accurate column space recovery can lead to more accurate matrix completion.

\subsection{Related Works}

There are two  natural ways to approach column space recovery with random sampling, which leads to two areas of related works: using the empirical covariance matrix, or using matrix completion results.

One approach, typically taken in the streaming PCA literature, is to assume that columns are i.i.d. and use the empirical covariance matrix of the columns to estimate the true covariance \citep{lounici2014high,gonen2016subspace,mitliagkas2014streaming}.  We can  then use the column space of this estimated covariance matrix.
This approach works, but it loses efficiency due to rescaling: for instance, if every entry is observed with probability $p$, then because each entry of the empirical covariance matrix is the product of two observed entries of the original matrix, each (off-diagonal) entry of the empirical covariance matrix is observed with probability $p^2$. Therefore, this approach pays a $p^{-2}$ penalty instead of $p^{-1}$ penalty in terms of missingness.
Moreover, while matrix completion approaches can have a $\log (\epsilon^{-1})$ dependence on the desired accuracy $\epsilon$ (in the low noise regime) for sample complexity, passing through the empirical covariance matrix naturally results in an $\epsilon^{-2}$ penalty \citep{lounici2014high, gonen2016subspace, mitliagkas2014streaming}.  Other work \citep{eftekhari2019streaming} in the streaming PCA literature avoids covariance estimation using a least squares approach (similar to us), but do not prove convergence to the true subspace.

Another approach would be to rely on powerful results in matrix completion (See,  for instance, \cite{candes2009exact,candes2010matrix, candes2009power,koltchinskii2011nuclear, recht2011simpler,cai2010singular,chatterjee2015matrix, jain2013low,hardt, keshavan2010matrix, keshavan2010noisymatrix, ge2016matrix}). However, there is no straightforward way to do this. For instance, one might think one could
first perform matrix completion on the partially observed matrix, and then use its singular value decomposition to recover the column space.
However, for an $N \times t$ matrix with $t> N$ whose rank is $r$, matrix completion results typically require more than $rt \log t$ observations.
Exceptions to the superlinear (in $t$) number of total observations \citep{krishnamurthy2013low,krishnamurthy2014power,balcan2016noise} violate our per-column budget or require a higher per-column budget for higher accuracy \citep{gamarnik2017matrix}.
This means that in order to get the desired guarantees from the matrix completion literature, we need to observe an \emph{increasing} number of entries per column. This is not a natural model for the budgeted learning case (there is no reason to assume that our budget increases with time) and is unnecessary, as we show in our theoretical results.
Another way to try to apply these matrix completion results is to split an $N \times t$ matrix into $N \times a$ matrices, with $a < t$, perform matrix completion on these smaller matrices (which now have enough samples), and then combine the resulting column space estimates. This might work if matrix completion were unbiased, but since the estimates tend to be the solution of a regularized problem, they tend to be biased (and bias correction is not simple \citep{javanmard2014confidence}).

As for active learning, there have been experimental results showing it can help  matrix factorization and completion \citep{elahi2016survey, he2009active, kawale2015efficient}, but they rarely come with theoretical guarantees.  \cite{kallus2016dynamic}, like us, consider a setting where customers are arriving with time, but their algorithm deviates from uniform sampling only for minimization of a bandit-like regret quantity, not for better estimation.
As mentioned above, \cite{krishnamurthy2013low, krishnamurthy2014power,balcan2016noise} prove theoretical results on matrix completion with active sampling, but they violate the budget assumption by sampling some columns in their entirety. \cite{gonen2016subspace} prove active sampling can help, but they share the drawbacks of using the first (covariance matrix estimation) approach and their estimate is drawn stochastically from a distribution (even after fixing the observations), resulting in error bounds that hold only in expectation, not with high probability.
\cite{chen2014coherent,chen2015completing} propose an active sampling strategy using leverage scores, which are similar to our active sampling strategy in that they are both derived from the column (and row) subspace.  However, their results are for exact matrix completion, and therefore does not provide theoretical guarantees in our setting where there is noise and the number of columns is growing.

While matrix completion results do not apply to our setting, in this work, we will leverage some of the technical components from that literature.
In particular, we show theoretically that alternating minimization will consistently recover the column subspace, both for uniformly random sampling and for active sampling.

\subsection{Organization}
The paper is organized in the following way: We first state the notation and assumptions (Section \ref{section:background}), followed by our algorithms (Section \ref{section:algorithms}). We then state our theoretical results (Section \ref{section:theory_results}) and present our experimental results (Section  \ref{section:experiments}). We conclude by mentioning ideas of the proof (Section \ref{section:proof_ideas}) followed by a brief summary (Section \ref{section:conclusion}).

\section{Background}
\label{section:background}
\textbf{Notation} For $M \in \Nat$, we use $[M]$ to denote $\{1,\ldots,M\}$ and for $M'\in \Nat$, $M' \leq M$, we use $[M':M]$ to denote $\{M', M'+1, \ldots, M\}$. For a matrix $Y \in \R^{N\times M}$, given $\Omega \subset [N] \times [M]$, a subset of indices (typically the indices of the observed entries), we define $\P_{\Omega}(Y) \in \R^{N \times M}$ by setting the entries with indices not in  $\Omega$ to $0$:
\begin{align*}
  (\P_{\Omega}(Y))_{ij} = \begin{cases}
    Y_{ij} & (i,j) \in \Omega     \\
    0      & (i,j) \notin \Omega.
  \end{cases}
\end{align*}
For $\Omega \subset [N] \times [M]$, $I \subset [M]$, we denote by $\Omega_I $ the set
$\{(n,m) \in \Omega \ \mid m \in I\}$.
We take complements of these sets by $\Omega_{I}^C := \{ (n,m) \in [N] \times I \ \mid \ (n,m) \notin \Omega_{I} \}$.
The singular value decomposition (SVD) of $Y$ expresses $Y$ as $U\Sigma V^T$, $U \in \R^{N \times r}, V \in \R^{M \times r}$, where $r$ is the rank of $Y$, and the columns of $U$ are orthornomal (known as the left singular vectors of $Y$), the columns of $V$ are orthonormal (the right singular vectors of $Y$), and $\Sigma$ is diagonal and contains the singular values. $\|\cdot\|_F$ is the Frobenius norm, given by $\|Y\|_F = \sqrt{ \sum_{n=1}^N \sum_{m=1}^M Y_{nm}^2}$. We use $\|\cdot\|$ to denote the operator norm, given by $\|Y\| = \sigma_1(Y)$, where $\sigma_1(Y) \geq \ldots \geq \sigma_{r}(Y)$ are the singular values of $Y$.
Throughout our paper, $t$ will denote the total number of columns of $Y_t \in \R^{N \times t}$ that are available, whereas $M \leq t$ is the second dimension of an ($N \times M$) submatrix we are considering at a particular point.

\subsection{Assumptions}
\label{section:assumptions}
Our goal is to estimate the column space of an approximately low rank matrix $Y_t \in \R^{N \times t}$ as the number of columns of the matrix grows. This is not possible for arbitrary growing matrices $Y_t$. As an extreme example, if all the columns after some point are identically zero, then we will no longer be able to learn anything about the column space, which means we need to assume that $\|Y_t\|$ is ``not too small". On the other hand, if $\|Y_t\|$ keeps growing too fast, we will only fit on the latest columns, which makes learning impossible, so we need $\|Y_t\|$ to be ``not too large".

First, we will assume that $Y_t$ arises from a low rank plus noise model. We will assume that the noise is actually Gaussian because we will use its rotational symmetry in the proofs. It is likely possible to relax this to more general classes of noise matrices, but we leave this for future work.
\begin{assumption}[Low rank plus Gaussian noise]
  \label{assumption:noise}
  $Y_t = \mathring{X} W_t^T + Z_t$, where $\mathring{X} \in \R^{N \times r}, W_t \in \R^{t \times r}$, and where $(Z_t)_{n,m} \overset{iid}{\sim} \N(0,\sigma_z^2)$.
\end{assumption}

Next, we need to make assumptions about $W_t$. Before stating the assumptions, we first define the ${\psi_2}$-norm.
\begin{definition}[$\psi_2$-norm]
  For a real valued random variable $A$, its $\psi_2$ norm is defined by
  \[
    \|A\|_{\psi_2} = \inf \{u > 0 \ : \ \E \exp(A^2/u^2) \leq 2\}.
  \]
\end{definition}
\begin{definition}[sub-Gaussian]
  We say that a real random variable $A$ is 1-sub-Gaussian if $\|A\|_{\psi_2} < 1$. We say that a random variable $B$ with values in $\R^N$ is 1-sub-Gaussian if  $\langle B, v \rangle$ is 1-sub-Gaussian for all $v \in \R^N$ with $\|v\|=1$.
\end{definition}

As alluded to previously, we have assumptions that control the growth of $W_t$ (and therefore $Y_t$) to be not too large and not too small. Because we have two phases the algorithm, initialization and iteration, we require two forms of these bounds. For initialization, our assumption is essentially the same as Assumption 1 of \cite{lounici2014high}.

\begin{assumption}[sub-Gaussian $W_t$]
  \label{assumption:subgaussian}
  For each $m \leq t$, each column $w_m \in \R^r$ of  $W_t$ satisfies:
  \begin{enumerate}
    \item $w_m$ is drawn independently (for each $m$) from a 1-sub-Gaussian distribution;
    \item there exists a numerical constant $c_1$ with $0 < c_1\leq 1 $ such that
          \[
            \E(\langle w_m, u \rangle)^2 \geq c_1 \|\langle w_m, u \rangle \|^2_{\psi_2} \ \forall u \in \R^r.
          \]
  \end{enumerate}
\end{assumption}

For iteration, we also need non-asymptotic bounds on the singular values, which would hold if $W_t$ were i.i.d. Gaussian from results from random matrix theory (see Corollary 5.35 of \cite{vershynin2010introduction}).

\begin{assumption}[Growth of Singular Values]
  \label{assumption:singular_vals}
  We assume that $\sigma_r(\X) >0$, and that there exists a $C_\sv$ large enough that for every $t \geq C_\sv$,
  $\X W_t$ satisfies
  \begin{align}
    \sigma_r(\mathring{X}W_t^T) \geq \frac{3}{4}\sigma_r(\mathring{X}) \sqrt{t}, \|\mathring{X}W_t^T\| \leq \frac{3}{2}\sigma_1(\mathring{X}) \sqrt{t}\label{eqn:sing_vals}
  \end{align}
  with probability at least $1-t^{-2}$ for $t \geq C_\sv$.
\end{assumption}

For matrix completion, we need an incoherence assumption as in \cite{candes2009power}, \cite{candes2009exact}, and \cite{recht2011simpler}.
There are many ways of interpreting this parameter, but intuitively, it says that observing an entry actually gives information about other entries.
It turns out that generating i.i.d. Gaussians for each entry of $W_M$ will produce right singular vectors that are incoherent: with $W_M = U_{W_M} \Sigma_{W_M} V_{W_M}^T$ the SVD, for some constants $C,c$, with probability at least $1-c M^{-3}\log M$, $\max_i \|P_{V_{W_M}} e_i \| \leq \sqrt{C \max\{r, \log M\}/M}$ (See Lemma 2.2 of \citep{candes2009exact}). Here $P_V$ denotes projection to the column space of $V$. This metric is equivalent to the coherence definition given below, which leads to Assumption \ref{assumption:incoherence}.

\begin{definition}
  The coherence of an $M \times r$ matrix $V$ is $\mu(V) := \max_{m \in [M]} (M/r) \|e_m^T V \|_2^2$.
\end{definition}
\begin{assumption}[Incoherence]
  \label{assumption:incoherence}
  There exists some $C_{\inc}$ such that  for large enough $M$, for any subset of $[t]$ of size $M$, with probability at least $1- M^{-3}\log M$, $\mu(V_{W_M}) \leq C_{\inc} \log M$.
\end{assumption}

Note we do not assume incoherence of the column space of $\mathring{X}$. In practice, having incoherent column space is probably helpful. But for our theoretical results, because $N$ is fixed as the number of columns $t$ is growing,
incoherence of $\X$, which provides high probability bounds with respect to $N$ (not $t$), are not as useful.

An example that satisfies all these assumptions is when each column $w_m \in \R^r$ of $W_t$  has entries that are distributed i.i.d. according to $\N(0, B)$ for some rank $r$ covariance matrix $B$.

\section{Algorithms}
\label{section:algorithms}

One way to view the column space of a matrix $Y \in \R^{N \times M}$ is to view it as
the span of the top $r$ eigenvectors of $YY^T$. We have $YY^T = \sum_{m=1}^M B_m$ where $B_m = y_m y_m^T$, and $y_m$ are the columns of $Y$. If we sampled each entry uniformly at random with probability $p$, we can get an estimate of each $B_m \in \R^{N \times N}$ in the following way: let $y'_m$ be the columns of $\P_{\Omega}(Y)$, and consider $B'_m = y'_m (y'_m)^T \in \R^{N \times N}$.
For independent Bernoulli($p$) sampling, if we form the matrix $D'_m : = p^{-2} B'_m  + (p^{-1} - p^{-2}) \diag(B'_m)$,
we have $\E[D'_m] = B'_m$. So if we approximate the eigenvectors of $\sum_{m=1}^M B'_m$, we might expect them to be close to the eigenvectors of $YY^T$ under mild assumptions.
This is the approach taken by \cite{gonen2016subspace} and \cite{mitliagkas2014streaming}. Indeed, under our assumptions, this will properly estimate the column subspace in expectation (Lemma 2 in \cite{gonen2016subspace}). If we exactly compute the eigendecomposition (which is computationally less efficient but has the best theoretical guarantees), we obtain \ScaledPCA (Algorithm \ref{alg:scaledPCA}), essentially the same as POPCA of \cite{gonen2016subspace}), whose pseudocode is included in the Appendix.
\footnote{\cite{lounici2014high} aims to estimate just the true covariance matrix, not the underlying subspace, under the setting where $t < N$.}

This is a nice and intuitive algorithm, but for matrix completion, it is known that methods based purely on spectral decompositions are
outperformed by methods based on optimization on the Frobenius norm of recovery error $\|\P_{\Omega}(Y-\hat Y)\|_F^2$ (such as least squares, gradient descent, or message passing) \citep{keshavan2012efficient}.  What is worse for \ScaledPCA is that because it estimates the covariance matrix first, it essentially pays a $p^{-2}$ penalty in terms of missingness instead of a $p^{-1}$ penalty.

In this work, we give a proof that alternating minimization (Algorithm \ref{alg:column_space}) can indeed be used to recover the column subspace.
Algorithm \ref{alg:column_space} performs spectral intialization followed by alternating minimization, using some of the samples ($\Omega^{(1)}$) to estimate $W$ and the remaining samples ($\Omega^{(2)}$) to estimate $X$.
Algorithm \ref{alg:column_space} uses two subroutines, \textsc{Sample}
and \textsc{MedianLS}
. \textsc{MedianLS} uses \SmoothQR \citep{hardt}, which is a version of QR factorization
that adds noise before performing QR, which for completeness, we include in Section \ref{subsection:smoothQR} of the Appendix. \SmoothQR helps maintains incoherence of the estimate of $W$ in \MedianLS
, and taking the median of estimates of $X$ leads to a higher probability bound, which are useful for our theory, but not necessary in practice \citep{hardt}.

We denote by $S \sim \Unif(\C(N,k))$
a subset $S \subset [N]$ that was sampled uniformly at random among subsets of $[N]$ of size $k$. In our algorithms, we assume we have enough columns to observe (e.g., for Algorithm \ref{alg:column_space}, $t \geq M_\init + s C^\med M \lceil \log M \rceil$). $C^\med$ is an absolute constant that is not required as input. $C_\inc$ is a constant from our incoherence assumption (Assumption \ref{assumption:incoherence}). We use $\triangleright$ to denote comments.

\begin{algorithm}[h]
  \caption{\textsc{ColumnSpaceEstimate} }
  \label{alg:column_space}
  \begin{algorithmic}[1]
    \Require{Partially observable $Y_t \in \R^{N \times t}$;  $ k^{(1)}, k^{(2)} \in \Nat$, such that the total number of samples per column is $k^{(1)} + k^{(2)}$; $M_\init \in \Nat$, the number of columns for initialization; $M \in \Nat$, the size of blocks of columns for least squares;  $s \in \Nat$, the number of blocks; $\epsilon$, the desired accuracy; $a$, a boolean indicator of active sampling}
    \Ensure{$\hat X \in \R^{N \times r}$, the column space estimate, $\Omega \subset [N] \times [t]$, the subset of observed indices}
    \Alg{ColumnSpaceEstimate}{$Y_t$, $k^{(1)}$, $k^{(2)}$, $M_\init$, $M$, $s$, $\epsilon$, $a$}
    \LineComment{Spectral initialization with uniform random sampling}
    \State{Initialize: $\Omega  \gets \emptyset$}
    \For{$m=1, \ldots, M_\init$}
    \State{$S \sim \Unif(\C(N,k^{(1)} + k^{(2)}))$}
    \State{$\Omega \gets \Omega  \cup (S \times\{m\} )$}
    \EndFor
    \State{$\hat X \gets $\ScaledPCA($\P_{\Omega}(Y_t), k^{(1)} + k^{(2)}, N$)}
    \LineComment{Least squares iteration}
    \State{$L \gets C^\med \lceil \log M \rceil$}
    \For{$i = 1,\ldots, s$}
    \LineComment{The next block of $L M$ columns to use, which further gets broken down into  $L$ blocks
      \linebreak \hspace*{4.1em} of size $M$ in \textsc{MedianLS}}
    \State{ $m \gets M^{\init}  + (i-1)LM + 1$}
    \State{ $I \gets [m: (m +  L M -1)]$}
    \State{$\Omega^{(1)}, \Omega^{(2)} \gets $\textsc{Sample}($\hat X, k^{(1)},k^{(2)}, I,a$)} \label{alg:line_sample}
    \State{$\hat X \gets$ \textsc{MedianLS} ($\hat X,Y_t, \Omega^{(1)}, \Omega^{(2)}$,
      $M$, $m$, $\epsilon$)} \label{alg:line_sample_2}
    \State{$\Omega \gets \Omega \cup \Omega^{(1)} \cup \Omega^{(2)}$}
    \EndFor
    \State \Return{$\hat X$, $\Omega$}
    \EndAlg
  \end{algorithmic}
\end{algorithm}

\begin{subroutine}[h]
  \caption{\textsc{Sample}: Choose samples for one block of columns}
  \label{alg:sample}
  \begin{algorithmic}[1]
    \Require{current estimate of column space $\hat X \in \R^{N \times r}$; $  k^{(1)}, k^{(2)} \in \Nat$, such that the total number of samples per column is $k^{(1)} + k^{(2)}$;  block of columns $I \subset [M]$; $a$, a boolean indicator of active sampling}
    \Ensure{$\Omega^{(1)}, \Omega^{(2)} \subset [N] \times I$, the samples for columns indexed by $I$}
    \Function{Sample}{$\hat X, k^{(1)}, k^{(2)}, I, a $}
    \State{Initialize: $\Omega^{(1)}  \gets \emptyset, \ \Omega^{(2)} \gets \emptyset$}
    \For{$m \in I$ }
    \LineComment{Choose each slice of $\Omega^{(1)}, \Omega^{(2)}$}
    \If{$a$}
    \LineComment{Use Equation \eqref{eqn:active} for active sampling}
    \State{$S^{(1)} \gets \Omega^*(\hat X; k^{(1)})\subset [N]$ }  \label{alg:active_step}
    \Else
    \State{$S^{(1)} \sim \Unif(\C(N,k^{(1)})) $ }
    \EndIf
    \State{$S^{(2)} \sim \Unif(\C(N,k^{(2)})) $ }
    \LineComment{Add the slices to $\Omega^{(1)}, \Omega^{(2)}$}
    \State{$\Omega^{(1)} \gets \Omega^{(1)} \cup (S^{(1)} \times \{m\})$}
    \State{$\Omega^{(2)} \gets \Omega^{(2)} \cup (S^{(2)} \times \{m\})$}
    \EndFor
    \State \Return{$\Omega^{(1)}, \Omega^{(2)}$}
    \EndFunction
  \end{algorithmic}
\end{subroutine}

\begin{subroutine}[h]
  \caption{\textsc{Median least squares}}
  \label{alg:ls_iteration}
  \begin{algorithmic}[1]
    \Require{Prior estimate $X^{\prev} \in \R^{N \times r}$; Partially observable $Y \in \R^{N \times t}$ and $\Omega^{(1)}, \Omega^{(2)} \in [N] \times [t]$ such that $\Omega^{(1)}(Y)$ and $\Omega^{(2)}(Y)$ are observed;  $M \in \Nat$, the block size to subdivide into for the median step; $m \in \Nat$, the beginning index of block of columns of size $C^\med M \lceil \log M \rceil$; $\epsilon$, the desired accuracy}
    \Ensure{$X \in \R^{N \times r}$, a column space estimate}
    \Function{MedianLS}{$X^{\prev}, Y, \Omega^{(1)},\Omega^{(2)} $, $M$, $m$, $\epsilon$}
    \State {$\tilde W_0 \gets \argmin{W' \in \R^{M \times r}} \|\P_{\Omega^{(1)}}(Y - X^{\prev}(W')^T)\|_F^2$} \label{alg:first_ls}
      \LineComment{QR factorization with added noise  for incoherence}
      \State{$\hatW, \tilde W, G_H^{(1)} \gets$ \SmoothQR$(\tilde W_0, \sigma_r(\X) \epsilon , C_\inc \log M)$} \label{alg:line_smoothqr}
      \State{$L \gets C^\med \lceil \log M \rceil$}
      \For{$i = 1, \ldots, L$ } \label{alg:line_median_0}
      \LineComment{Get the next block of $M$  columns to use for median}
      \State{$J_i\gets [(m+  (i-1)M) : (m + iM-1)]$}
      \State{$\tilde X^{(i)} \gets \argmin{X' \in \R^{N \times r}} \|\P_{\Omega^{(2)}_{J_i}}(Y - X'(\hat W)^T)\|_F^2$} \label{alg:second_ls}
    \EndFor
    \State{$\tilde X \gets$ elementwise median of
      $\{\tilde X^{(1)}, \ldots, \tilde X^{(L)}\}$} \label{alg:line_median}
    \State{$X \gets$ Orthonormal basis of column space of $\tilde X$}
    \State \Return{$X$}
    \EndFunction
  \end{algorithmic}
\end{subroutine}

\paragraph{Practical Considerations} We state our algorithms in a way that is natural to prove theoretical results, which is the main goal of this paper. However, for more practical purposes, the large block size $M$ might at first seem prohibitive to use in \MedianLS. We mitigate this in the following way: first, as mentioned above, the \SmoothQR step in Line \ref{alg:line_smoothqr} of \MedianLS  and median step in Line \ref{alg:line_median} are not necessary in practice. Therefore, given an $X^\prev$, we need only to perform two linear least squares regressions (lines \ref{alg:first_ls} and \ref{alg:second_ls}). The first regression (line \ref{alg:first_ls}), which fits $\tilde W$,  can be done separately for each column. The second regression, which fits $X$ (line \ref{alg:second_ls}), can be performed in an online manner. Two possible options are to perform least squares recursively (which gives exactly the same result as doing a batch linear least squares), or to do gradient descent (which is more practical).

Both of these options process one column at a time (instead of processing it as a block as in Lines \ref{alg:line_sample} and \ref{alg:line_sample_2} in Algorithm \ref{alg:column_space}), and lead to time and space complexity that is linear in the block size $M$.

\paragraph{Active Sampling} Our proof naturally leads to an active sampling strategy that can help subspace recovery, as confirmed in our experiments. Each iteration of fitting a $\tilde W$ (Line 3 of \MedianLS) is a linear least squares regression, whose estimation error decreases as the minimum singular value of the design matrix increases. Therefore, a good candidate strategy for \Sample is to choose the rows of $X^\prev$ to maximize the minimum singular value of the induced submatrix. More precisely, for  $S = \{s_1,\ldots,s_k\} \subset [N]$, we define $\Q_{S}$ as the operator that projects the $N \times r$ matrix to a $k \times r$ matrix specified by $[\Q_S(X)]_{ij} = X_{s_i, j}$. (The objective in Equation \eqref{eqn:active} is invariant to the ordering chosen on $S$.) Given an estimate $\hat X$, our active sampling chooses
\begin{align}
  {\textstyle \Omega^*(X;k^{(1)}) = \argmax{S \subset [N], |S|=k^{(1)}} \ \sigma_r(\Q_{S}(X)),
  }
  \label{eqn:active}
\end{align}
as $S^{(1)} \subset [N]$.
We will need other samples of rows of $Y$ to estimate $X$ from this estimated $\hatW$, and we choose these samples randomly, so we can get equal informations about every row of $X$, i.e., $S^{(2)}$ is chosen uniformly at random.
\section{Theoretical Results}
\label{section:theory_results}

\paragraph{Budget per column} In the following theorems, we will assume that $k^{(1)} \geq r$, and $k^{(2)} \geq 1$.  We need $k^{(1)}$ to be at least $r$ because we observe $k^{(1)}$ entries per column for Line \ref{alg:first_ls} of \MedianLS. However, $k^{(2)}$ need not be as large because as the number of columns tends to infinity, we will observe at least $r$ entries in each row. Therefore, the total number of required samples is only $r+1$ per column. But we do not recommend setting $k^{(2)}$ as low as 1 in practice, especially without sample splitting.

\paragraph{Subspace Recovery Metric} For our theorem statements, we let $U \in \R^{N \times r}$ be the matrix whose orthonormal columns are the left singular vectors of $\mathring{X}$. In general, when we compute the SVD of $\mathring{X} W_t = U_t \Sigma_t V_t^T$, the resulting $U_t$ might not contain the same singular vectors as $U$, but they span the same subspace. We use a distance measure on subspaces that does not depend on such representations, namely the largest principal angle between subspaces. This can be  defined for two matrices with orthonormal columns $U, X \in \R^{N \times r}$ by $\sin\theta(X,U) = \|(I_N - XX^T)U\|$  \citep{zhu2013angles}. Note
$\sin \theta(U,UO) = 0$ for any  orthogonal matrix $O \in \R^{r \times r}$.

\paragraph{Initialization} The initialization conditions are quite stringent in theory, but in practice, as has been empirically\footnote{For much higher sampling complexity and Bernoulli samples, it has been shown theoretically by \cite{ge2016matrix} and \cite{ge2017no}.}
shown in other optimization approaches, only mild initialization can suffice. This is consistent with our own experiments in Section \ref{section:experiments}.

Proofs of all theorems may be found in the Appendix (Section \ref{section:thm_proofs}). For ease of notation, we define $q^{(1)} :=\frac{k^{(1)}-r+1}{r(N-k^{(1)})+k^{(1)}-r+1}$. Note that $\frac{1}{rN} \leq q^{(1)}\leq \frac{k^{(1)}}{r(N-k^{(1)})}$, and that $\frac{k^{(1)}}{q^{(1)}}$ is a decreasing quantity with respect to $k^{(1)}$.
In order to simplify our bounds a little, we will additionally assume that $k^{(1)} \leq \frac{N}{2}$, which implies that $q^{(1)} \leq \frac{2k^{(1)}}{rN}$.

\subsection{Active Sampling}

Noise in observations presents an obstacle to recovering the column space, and if the noise variance is too large compared to the $r$-th singular value of $\mathring{X}$, then it can drown out this `signal' in the noise when performing alternating minimization. Therefore, we impose Assumption \ref{assumption:noise_var} or \ref{assumption:noise_var_random} to ensure that we have enough signal.

\begin{assumption}[Size of Noise for Active Sampling]
  \label{assumption:noise_var}
  $\sigma_Z \leq \frac{1}{48}\frac{\sqrt{q^{(1)}}}{\sqrt{k^{(1)}}}\sigma_r(\mathring{X})$.
\end{assumption}

There are two factors that influence the rate of convergence. One factor is that we only have partial observations. The other factor is that we have noise in our observations.  When $\sigma_Z$ is small compared to the desired accuracy $\epsilon$,
\begin{align}
  {\textstyle       \sigma_z  \sqrt{k^{(1)}} \leq \epsilon \sigma_1(\X)\sqrt{r}, }\label{eqn:epsilon_large}
\end{align}
the effect of having only partial observations dominates.
For instance, this is true when observations do not contain noise. \footnote{This holds in theory, only up to $e^{-O(M)}M$, where $M$ is the blocksize, because of a technicality in our smooth orthonormalization step.}
When  $\sigma_Z$ is large compared to the desired accuracy $\epsilon$,
\begin{align}
  {\textstyle       \sigma_z  \sqrt{k^{(1)}} \geq \epsilon \sigma_1(\X)\sqrt{r}, }\label{eqn:epsilon_small}
\end{align}
the effect of noise dominates. Therefore, we prove different convergence rates for each regime.

\begin{theorem}[Active sampling, for small $\sigma_Z/\epsilon$]
  \label{thm:noisy_active_large_epsilon}
  Suppose Assumptions \ref{assumption:noise}, \ref{assumption:noise_var}, \ref{assumption:subgaussian}, \ref{assumption:singular_vals}, \ref{assumption:incoherence} hold, $N/2 \geq k^{(1)} \geq r,k^{(2)} \geq 1$, $1 \geq \sigma_r(\X) \epsilon $, and
  Equation \eqref{eqn:epsilon_large} holds.
  Then there exist constants $C_{\ref{thm:noisy_active_large_epsilon}}^\init$, $C_{\ref{thm:noisy_active_large_epsilon}}^\iter$, $C_{\ref{thm:noisy_active_large_epsilon}}^\prob, C^{(Q)}$ such that, if we initialize with $M_\init$ columns, where

  \begin{align}
    {\textstyle
      M_\init \geq C_{\ref{thm:noisy_active_large_epsilon}}^{\init}
      \frac{\sigma_1(\X)^6 N^2 (\log M_{\init})^3r^2 }{\sigma_r(\X)^6 (k^{(1)} + k^{(2)})^2  q^{(1)}}
      ,}
    \label{eqn:large_eps_init}
  \end{align}
  and we use $s$ blocks, where
  \begin{align}
    {\textstyle    \mathrm{s} \geq \log_2\left( \frac{\sigma_r(\X)\sqrt{q^{(1)}}}{48 \sigma_1(\X) \sqrt{r}\epsilon}\right), }\label{eqn:large_eps_numblocks}
  \end{align}
  and each block has size $M$, with
  \begin{align}
    {\textstyle    M \geq  C_{\ref{thm:noisy_active_large_epsilon}}^\iter
      \frac{\sigma_1(\X)^6 r^3 N (\log M)^2}{\sigma_r(\X)^6 k^{(2)} q^{(1)}} + \log \left( \frac{1}{\epsilon}\right),
    }
    \label{eqn:large_eps_blocksize}
  \end{align}
  and
  \begin{align}
    \sigma_r(\X)\epsilon \geq e^{-C^{(Q)} M} M,
  \end{align}
  then  \ColSpaceEst($Y_t$,  $k^{(1)}, k^{(2)}$,
  $M_\init$,
  $M$,
  $s$,
  $\epsilon$,
  True) returns an $\hat X$ such that
  $\sin\theta (U,\hat X) \leq   \epsilon$
  with probability at least $1-2 M_\init^{-2}-C_{\ref{thm:noisy_active_large_epsilon}}^\prob s M^{-2}$.
\end{theorem}

Whenever Theorem \ref{thm:noisy_active_large_epsilon} holds, the sample complexity grows only logarithmically with $\epsilon^{-1}$,
which is a feature of a matrix completion approach (versus a spectral approach, which always has a dependence of $\epsilon^{-2}$) in the small $\sigma_Z/\epsilon$ regime.

When the $\sigma_Z$ is large compared to the desired accuracy $\epsilon$, we can get a $\sigma_Z$-dependent bound, with a $\epsilon^{-2}$ dependence on desired accuracy $\epsilon$. The initialization step for this regime  consists of Algorithm \ref{alg:column_space} instead of a spectral initialization. The full pseudocode for \textsc{DoubleColumnSpaceEstimate} (Algorithm  \ref{alg:column_space_metablocks})
can be found in the Appendix.

\begin{theorem}[Active sampling, for large $\frac{\sigma_Z}{\epsilon}$]
  \label{thm:noisy_active_small_epsilon}
  Suppose Assumptions \ref{assumption:noise}, \ref{assumption:noise_var},\ref{assumption:singular_vals}, \ref{assumption:incoherence} hold.
  Then there exist constants $C_{\ref{thm:noisy_active_small_epsilon}}^{\init}, C_{\ref{thm:noisy_active_small_epsilon}}^{\iter},
    C_{\ref{thm:noisy_active_small_epsilon}}^\prob, C^{(Q)}$ such that for
  $N/2 \geq k^{(1)} \geq r,k^{(2)} \geq 1$, $1 \geq \sigma_r(\X) \epsilon \geq e^{-C^{(Q)} \min(M_1,M_2)}\min(M_2, M_1)$, and $\epsilon$ satisfying equation \eqref{eqn:epsilon_small}.

  if we initialize with $M_\init$ columns, where
  \[
    {\textstyle
        M_\init \geq C_{\ref{thm:noisy_active_small_epsilon}}^\init
        \frac{\sigma_1(\X)^6 N^2 (\log M_{\init})^3r^2 }{\sigma_r(\X)^6 (k^{(1)} + k^{(2)})^2  q^{(1)}}
      }
  \]
  and perform alternating minimization with
  $  s_1 \geq \log \left( \frac{\sigma_1(\X) \sigma_r(\X)\sqrt{q^{(1)}}}{48 \sigma_Z \sqrt{k^{(1)}}}\right) $
  blocks of size
  \begin{align*}
    {\textstyle    M_1 \geq  C_{\ref{thm:noisy_active_large_epsilon}}^\iter
      \frac{\sigma_1(\X)^6 r^3 N (\log M+1)^2}{\sigma_r(\X)^6 k^{(2)} q^{(1)}} + \log\left(\frac{1}{\epsilon}\right),
    }
  \end{align*}
  followed by alternating minimization with $s_2 = 1$ block of size
  \begin{align*}
    {\textstyle
      M_2 \geq C_{\ref{thm:noisy_active_small_epsilon}}^{\iter}
      \max \left\{\frac{r^2\sigma_Z^2 \sigma_1(\X)^4 N k^{(1)} (\log M_2)^2}{\sigma_r(\X)^6 k^{(2)} q^{(1)}\epsilon^2},  \frac{r\sigma_1(\X)^2 \sqrt{N} \log M_2}{\sigma_r(\X)^2\sqrt{k^{(2)}}} \right\} + \log\left(\frac{1}{\epsilon}\right)
    },
  \end{align*}
  then \textsc{DoubleColumnSpaceEstimate}($Y_t$, $k^{(1)}$, $k^{(2)}$, $M_\init$, $M_1$, $M_2$, $s_1,s_2$, $\epsilon$, True)
  returns an $\hat X$ such that
  $\sin \theta (U,\hat X) \leq   \epsilon$
  with probability at least $1 - 2 M_\init^{-2}-C_{\ref{thm:noisy_active_large_epsilon}}^\prob s_1 M_1^{-2} -
    C^\prob_{\ref{thm:noisy_active_small_epsilon}}M_2^{-2} $.
\end{theorem}

\paragraph{Comparison with \ScaledPCAEstimate}
We compare with the theoretical results from using the {\ScaledPCAEstimate} approach with Proposition 3 from \cite{lounici2014high}, as
\cite{gonen2016subspace} show theorems in a different setting, use a different metric, and prove bounds only in expectation.
For simplicity, we will omit dependence on the condition number (assume $\sigma_1(\X) = \sigma_r(\X) = 1$) and assume that $k^{(1)} = k^{(2)}=:k$ .
When $\sigma_Z/\epsilon$ is small (Equation \eqref{eqn:epsilon_large}), Theorem \ref{thm:noisy_active_large_epsilon}'s  logarithmic dependence on $\epsilon^{-1}$ is better than the $\epsilon^{-2}$ dependence of \cite{lounici2014high}, but the dependence on $r$ and $k$ is worse, by $r^3k$.
When $\sigma_Z/\epsilon$ is large (Equation \eqref{eqn:epsilon_small}, Theorem \ref{thm:noisy_active_small_epsilon}), our sample complexity needs $\tilde O(rk/N)$ as many samples as \cite{lounici2014high}, which can be fairly small.

\subsection{Uniformly random sampling}

When we use random sampling, there is a chance per column that we might choose a ``bad" subset, which is small with respect to $N$, but does not change with respect to $M$. Since we need to avoid ``bad" subsets for all $M$ columns, in the regime of $M \gg N$, this would give us an unacceptable probability of failure in theory, though in practice, this probably does not occur.
Therefore, we assume that the true $\X$ has no ``bad" subsets and use a longer initialization period to ensure that our $\hat X$ also has no ``bad" subsets. When $\X \in \R^{N \times r}$ has rank $r$ (which is true by Assumption \ref{assumption:singular_vals}), the assumption about the absence of ``bad" subsets is equivalent to the $k$-isomeric condition by \cite{NIPS2017_6680}.
\begin{definition}[$k$-isomeric \citep{NIPS2017_6680}]
  A matrix $X \in \R^{N \times r}$ is called $k$-\emph{isomeric} if and only if any $k$ rows of $X$ can linearly represent all rows in $X$.
\end{definition}
We define the smallest singular value of any $k^{(1)}$ rows of a matrix $X \in \R^{N \times r}$, which is the opposite of the desired criterion in Equation \eqref{eqn:active}.
\begin{align}
  \sigma_*(X;k^{(1)}) :=  \min_{S \subset [N], |S| = k} \sigma_r(\Q_S(\X)) \label{eqn:worst_omega}
\end{align}
Assuming that $U$ has rank $r$, if $\X$ is $k^{(1)}$-isomeric, $\sigma_*(U;k^{(1)}) > 0$.

We note that every $N \times r$ matrix $X$ with orthogonal columns has $\sigma_*(X;k^{(1)}) \leq \sqrt{p^{(1)}}$ by Lemma \ref{lemma:worst_omega}, and in fact, $\sigma_*(X;k^{(1)})$ could be arbitrarily small. For random sampling, $\sigma_*(U;k^{(1)})$ will play (up to a constant term) the same role as $\sqrt{q^{(1)}}$ in active sampling, for instance, in the bound on the noise variance.

\begin{assumption}[Size of Noise for Random Sampling]
  \label{assumption:noise_var_random}
  $\sigma_Z \leq \frac{1}{96}\frac{\sigma_*(U;k^{(1)})}{\sqrt{k^{(1)}}}\sigma_r(\mathring{X})$.
\end{assumption}

The difference in sampling complexity in active versus random sampling is the difference between $(\sigmaU)^2$ and $q^{(1)}$. Theorems \ref{thm:noisy_active_large_epsilon} and \ref{thm:noisy_active_small_epsilon} still hold with exactly the same proof if we replace $q^{(1)}$ with
$(\max_{|S| = k^{(1)}} \sigma_r(U_S))^2$. With this replacement, the corresponding bound for the active learning case will always be better than the bound for the noisy case. For instance, because there is, in general, no lower bound for  $\sigma_*(U;k^{(1)})$, we cannot give an upper bound on the initialization step of random sampling that holds independent of $\X$, which is something we \emph{can} do in the case of active sampling. The full statements and proofs for the theorems for the uniformly random sampling case (Theorems \ref{thm:noisy_random_large_epsilon} and \ref{thm:noisy_random_small_epsilon}) can be found in the Appendix.

\section{Experiments}
\label{section:experiments}
\begin{figure*}[h]
  \centering

  \begin{subfigure}[b]{0.4\textwidth}
    \includegraphics[width=1.1\textwidth]{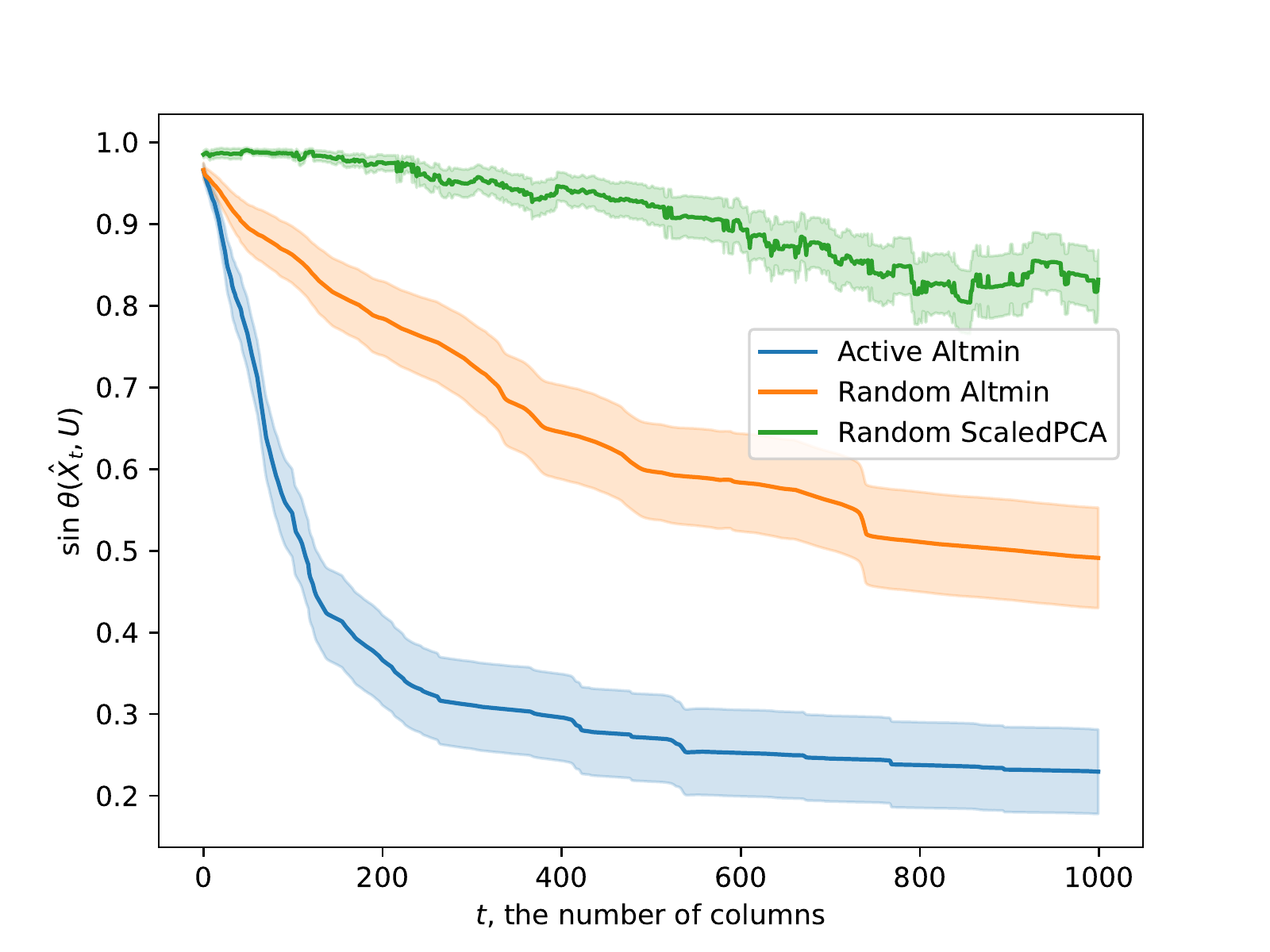}
    \caption{Simulated data, $X$ recovery}
    \label{fig:simulated_X}
  \end{subfigure}
  \begin{subfigure}[b]{0.4\textwidth}
    \includegraphics[width=1.1\textwidth]{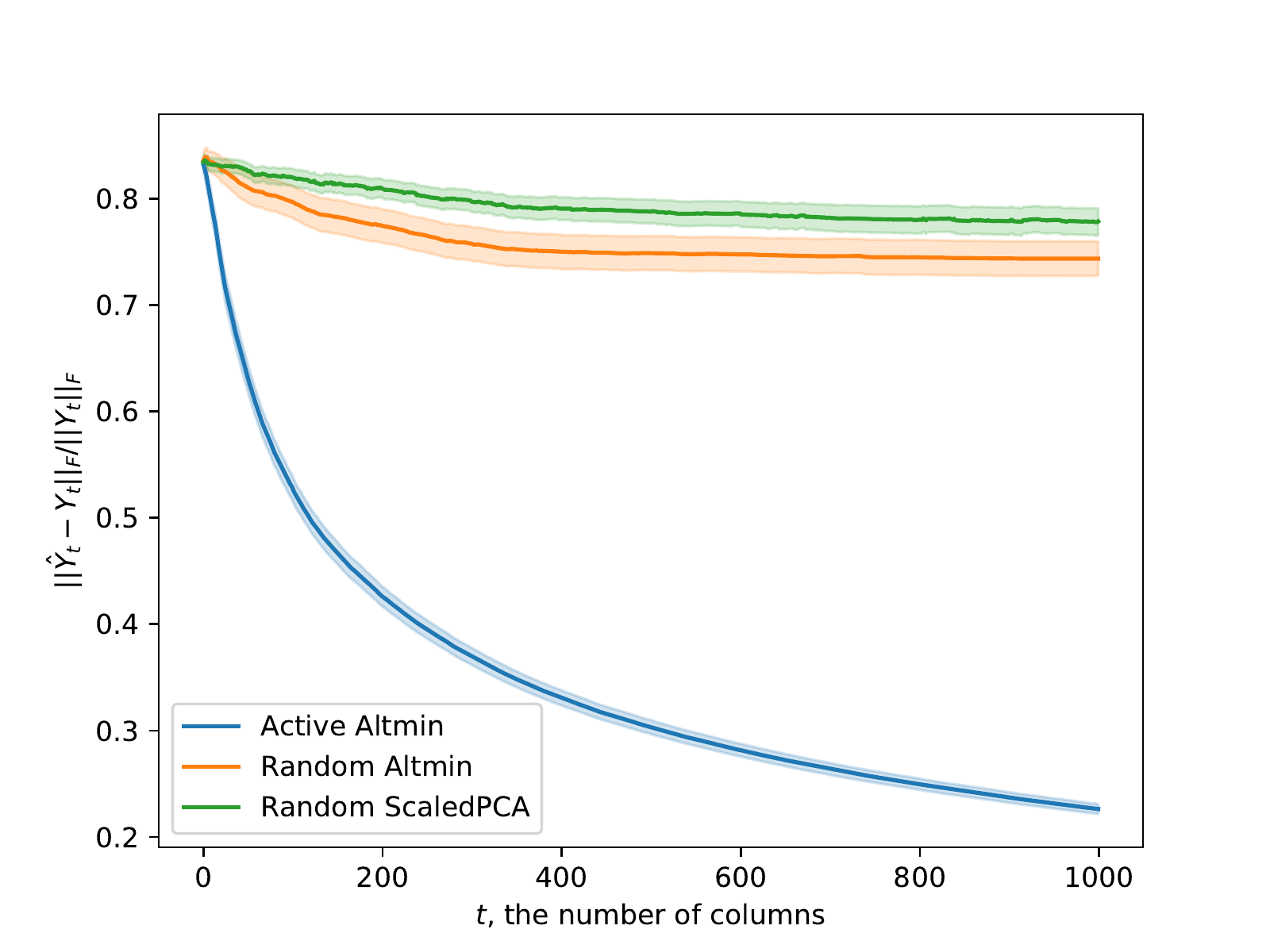}
    \caption{Simulated data, $Y$ recovery}
    \label{fig:simulated_Y}
  \end{subfigure}

  \begin{subfigure}[b]{0.4\textwidth}
    \includegraphics[width=1.1\textwidth]{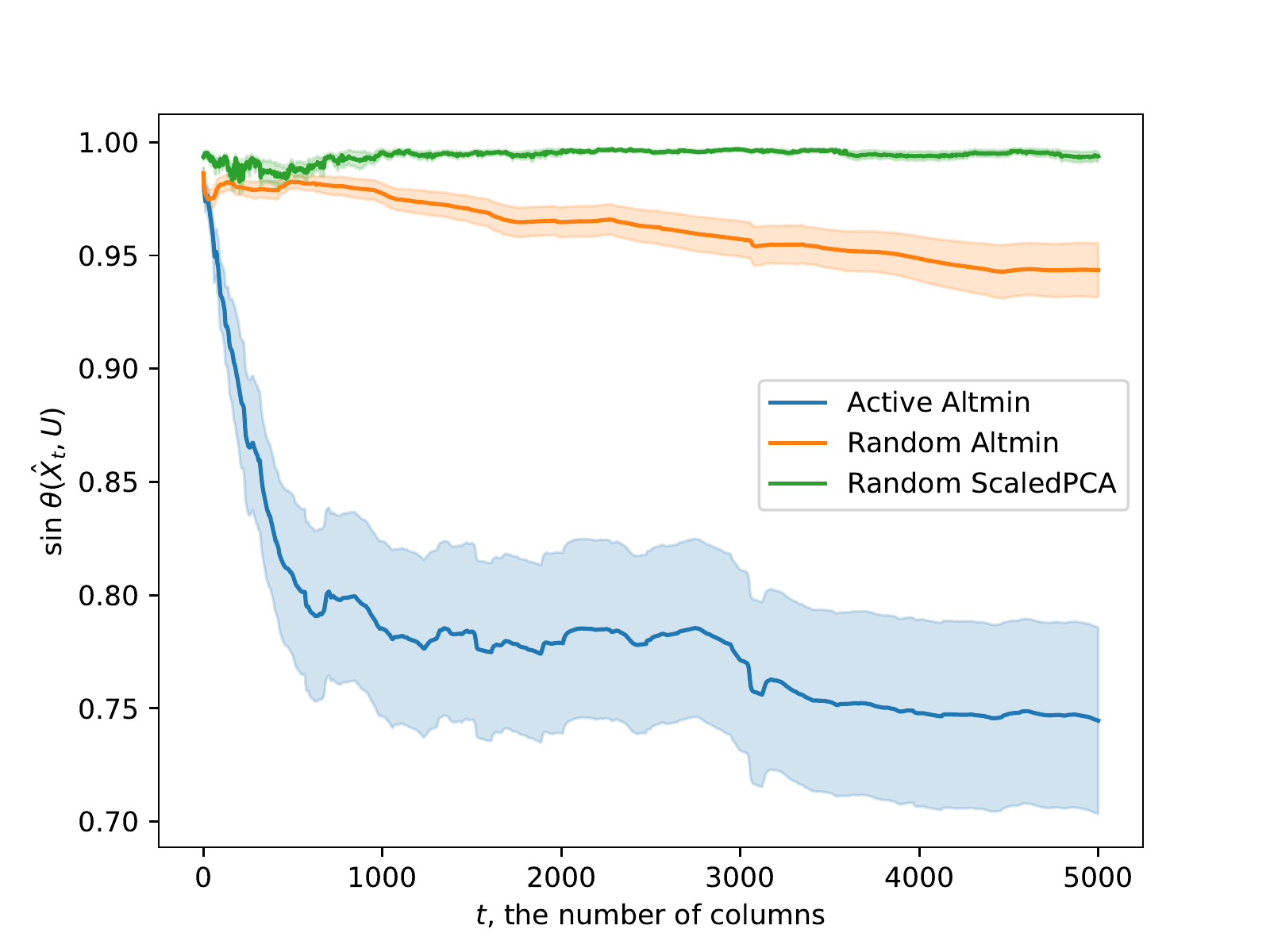}
    \caption{MIMIC II data, $X$ recovery}
    \label{fig:mimic_X}
  \end{subfigure}
  \begin{subfigure}[b]{0.4\textwidth}
    \includegraphics[width=1.1\textwidth]{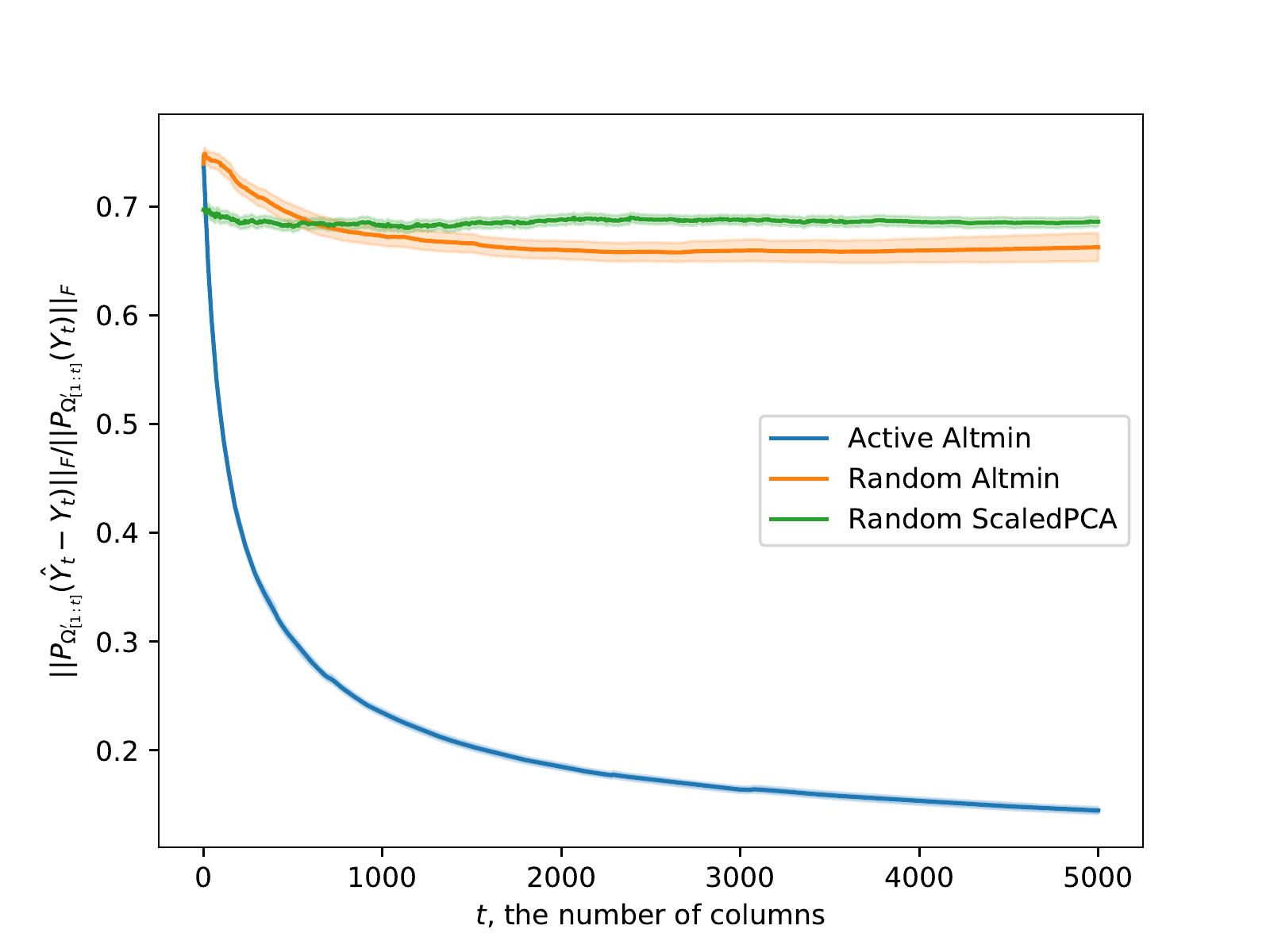}
    \caption{MIMIC II data, $Y$ recovery}
    \label{fig:mimic_Y}
  \end{subfigure}
  \caption{Error versus number of columns $t$}
  \label{fig:figs_1}
\end{figure*}
\paragraph{Synthetic data} For the synthetic data experiments, we use the model from Assumption  \ref{assumption:noise} with i.i.d. Gaussian columns. That is, for each simulation, we generate a fixed $\mathring{X} \in \R^{N \times r}$, and we generate the $t$-th column by $\mathring{X} w_t + z_t$, where $w_t, z_t \in \R^{r \times 1}$ and $w_t \sim \N(0,I_r)$, $z_t \sim \N(0,\sigma_Z^2I_r)$.
Since we do not require $\X$ to be incoherent (which would result from light tailed distributions), we use a heavy tailed distribution (specifically the standard Cauchy distribution) to generate each entry of $\X$ independently. We set $\sigma_Z =0.1$, $N=50$, and $r =  6$.

\paragraph{MIMIC data} For our real data experiments, we use the MIMIC II dataset,  which contains data for ICU visits at Beth Isreael Deaconess Medical Center in Boston, Massachusetts between 2001 and 2008 \citep{mimic}. We focused on patients aged 18-89 (inclusive) who were having their first ICU visit, and who stayed in the ICU for at least 3 days. For these patients (columns), we took 1269 features which mostly include lab test results.
Because the data has many missing entries, we restricted the data to those columns and rows that had less than 50\% missing entries, which led to 115 covariates (rows) and 14584 patients (columns).
Then, for each run, we randomly chose a submatrix of $N=50$ covariates and  $t=5100$ patients, and we use $r=6$ as in the simulated data. To evaluate column space recovery , we estimated a ``ground truth" $\X$ using SVD on our data, with missing values replaced by zeros. However, when evaluating $Y_t$ recovery, we only measure error on the non-missing values (i.e., those that were present in the data, which is a strict superset of those that were observed by the algorithms).

\paragraph{Approximately active greedy sampling} We choose a fixed number $k = k^{(1)} + k^{(2)}$ to sample per column.
For active sampling, we set $k^{(1)} = k^{(2)} = 6$, and for random sampling, we set $k = 12$, so that both strategies observe the same number of samples per column. Ideally, our active sampling method would choose the subset $S^{(1)}$ of size $k^{(1)}$ that satisfies Equation \eqref{eqn:active}.
However, since exhaustive search is computationally infeasible, we use an efficient method that approximates this optimization, namely, Algorithm 1 from \cite{avron2013faster} . This algorithm produces an $S^{(1)}$ such that
$\sigma_r(\Q_{S^{(1)}}(X^\prev))\geq \sqrt{\tilde q^{(1)}}$, where $\tilde q^{(1)} = \frac{k^{(1)}-r+1}{r(N-r+1)}$. $\tilde q^{(1)}$ is greater than $q^{(1)}$, but has a similar behavior as $q^{(1)}$ for small $k^{(1)}$. Analogues of Theorems
\ref{thm:noisy_active_small_epsilon} and \ref{thm:noisy_active_large_epsilon}, with $q^{(1)}$ replaced by $\tilde q^{(1)}$, hold when we use this approximation algorithm for active sampling.

\paragraph{Deviation from theoretical assumptions}
Our recovery methods operate in a more practical setting than our theory requires. For alternating minimization, the initialization uses much fewer columns than our theorems require, we do not do sample splitting, we do not fix the time horizon beforehand, and we update $\hat X$ as we partially observe each column. This continual updating means that even if we chose $S^{(1)}$ at time $t_0$ such that $\Q_{S^{(1)}}(\hat X_{t_0})$ was large, when we use it at some timestep $t_1 > t_0$, $\Q_{S^{(1)}} (\hat X_{t_1})$ may not be large. We also skip the SmoothQR and Median steps and add L2 regularization with $\lambda = 0.05$ for stabilization.

\paragraph{Matrix recovery} In many cases, the reason that we care about recovering subspaces accurately is so that we can recover the original matrix $Y_t$ accurately. Therefore, we also measure matrix recovery. Given an estimate of the column subspace $\hat X$, the corresponding estimate $\hat Y_t$ is computed by imputing the missing entries by taking the best regularized least-squares fit over the observed entries: $\P_{\Omega_{[1:t]}^C}(\hat Y_t) = \P_{\Omega_{[1:t]}^C}(\hat X \beta^*)$, where $\beta^*=\argmin{\beta}\|\P_{\Omega_{[1:t]}}(\hat X \beta - Y_t)\|_F^2 + 0.05 \|\beta\|_F^2$. The algorithms do not have to fit the entries that it has observed, i.e., $\P_{\Omega_{[1:t]}}(\hat Y_t) = \P_{\Omega_{[1:t]}}(Y_t)$.

\subsection{Results}

Figure \ref{fig:figs_1} shows the results of our simulations, averaged over 50 runs.
Our active sampling method samples $k^{(1)}$ entries as described above (approximately active greedy sampling) and $k^{(2)}$ samples uniformly at random.
We compare three methods: \ScaledPCA (green), alternating minimization with uniformly random sampling (orange), and alternating minimization with active sampling (blue).
We denote by $\hat X_t$ and $ \hat Y_t$ the estimates of $X$ and $Y$ after observing $t$ columns. We perform the initialization step with 100 columns, and plot the error as additional columns are observed, for 1000 additional columns for the simulated data and 5000 additional columns for the MIMIC II data. We indicate standard error through shading. In Figures \ref{fig:simulated_X} and \ref{fig:mimic_X}, the error is the sine of the largest principal angle between two subspaces, as discussed in Section \ref{section:theory_results}, and in Figures \ref{fig:simulated_Y} and \ref{fig:mimic_Y}, we use the normalized matrix recovery error, which is given by $\frac{\|\hat Y_t - Y_t\|_F}{\|Y_t\|_F}$, for the simulated data. Since we do not know all the entries of the MIMIC II dataset, we use  $\frac{\|\P_{\Omega'_{[1:t]}} (\hat Y_t - Y_t )\|_F}{\|\P_{\Omega_{[1:t]}'} (Y_t)\|_F}$, where $\Omega'$ consists of the entries for which we have ground truth in the dataset (many of which were not observed by the algorithms).
\paragraph{Column space recovery} Figures \ref{fig:simulated_X} and \ref{fig:mimic_X} show that alternating minimization (both random and active sampling) recovers the column space more accurately than  \ScaledPCA. Furthermore, when using alternating minimization, using active samples results in a lower column space recovery error than using  uniformly random samples.
\paragraph{Matrix recovery} In Figures \ref{fig:simulated_Y} and \ref{fig:mimic_Y}, we can see that when algorithms have more accurate column space estimates, the corresponding matrix estimate $\hat Y_t$ also tends to be more accurate. In Figure \ref{fig:mimic_Y}, for the first few hundred columns, alternating minimization with random sampling has a less accurate matrix estimate $\hat Y_t$ than \ScaledPCA. However, this is only when alternating minimization with random sampling has a poor column space estimate (though still slightly better than that of \ScaledPCA). Moreover, the relative performance of alternating minimization with random sampling improves (both for matrix and column space recovery) as the number of observed columns grows, which is the setting of our theoretical results.
Also, note that alternating minimization with active sampling always performs better than \ScaledPCA.

\section{Ideas of the Proof}
\label{section:proof_ideas}

Each iteration of alternating minimization involves optimizing $\hat W \in \R^{M \times r}$ given a fixed $\hat X^\prev \in \R^{N \times r}$, and then optimizing $\hat X$ given this $\hat W$.

\cite{jain2013low} and \cite{hardt} argue that each minimization step is similar to performing a step in in the power method (e.g., finding the top eigenvector of a symmetric matrix $A$ by setting $x_{t+1} = Ax_t/\|Ax_t\|_F$). In their setting, $\tan \theta(\hat W, V) \leq \tan \theta(\hat X^\prev, U)$ and $\tan \theta(\hat X, U) \leq \tan \theta(\hat W, V)$, leading to successively better estimation, $\tan \theta (\hat X, U) \leq \tan \theta (\hat X^\prev, U)$, with each iteration. (Here, $U$ and $W$ represent the row subspace and column space, respectively, of the de-noised version of $Y$.)

In our setting, because of the asymmetry between $N$ and $M$,  $\tan \theta(\hat W, V) \leq \tan \theta(\hat X^\prev, U)$ no longer holds. However, it remains true that $\tan \theta(\hat X, U) \leq \tan \theta(\hat W, V)$. Furthermore, it turns out that by adjusting the block size $M$ appropriately, we can make this decrease be large enough to compensate for the increase from $\tan \theta(\hat X^\prev, U)$ to $\tan \theta(\hat W, V)$.  In a way, this is in the spirit of averaging multiple estimates of the column subspace, by first passing through $\hat W$, and collecting information from enough columns of $\hat W$ to gain a more accurate estimate.

In the small $\sigma_z/\epsilon$ regime, this decrease from $\tan \theta(\hat X, U)$ to $\tan \theta(\hat X^\prev, U)$ is actually multiplicative, leading to exponential convergence in the number of iterations.
\section{Conclusion}
\label{section:conclusion}

In this work, we proved that an alternating minimization approach to estimating the column subspace of a partially observed matrix succeeds -- as the number of columns grows, we can estimate the column space to any given accuracy with probability tending to 1. We showed theoretically and experimentally that this approach works better than the naive one that performs PCA on the elementwise rescaled empirical covariance matrix.  We also showed that using some number $k^{(1)} \geq r$ of actively chosen samples in addition to random samples outperforms random sampling.
\bibliography{expandedpaper}
\vfill
\pagebreak

\appendix

\section{Algorithm for Two Block Sizes and Uniformly Random Sampling Theorems}
\label{section:meta_alg}

\begin{algorithm}[H]
  \caption{\textsc{DoubleColumnSpaceEstimate}: column space estimation with two block sizes}
  \label{alg:column_space_metablocks}
  \begin{algorithmic}[1]
    \Require{Partially observable $Y_t \in \R^{N \times t}$;  $ k^{(1)}, k^{(2)} \in \Nat$, such that the total number of samples per column is $k^{(1)} + k^{(2)}$; $M_\init \in \Nat$ the number of columns for initialization; $M_1, M_2 \in \Nat$, the sizes of blocks of columns for least squares;  $s_1,s_2 \in \Nat$ the numbers of blocks; $\epsilon$, the desired accuracy; $a$, a boolean indicator of active sampling}
    \Function{DoubleColumnSpaceEstimate}{$Y_t$, $k^{(1)}$, $k^{(2)}$, $M_\init$, $M_1, M_2, s_1, s_2, \epsilon, a$}
    \LineComment{Spectral initialization with uniformly random sampling}
    \State{$\Omega_{M_\init}  \gets \emptyset$}
    \For{$m=1, \ldots, M_\init$}
    \State{$S \sim \Unif(\C(N,k^{(1)} + k^{(2)}))$}
    \State{$\Omega \gets \Omega  \cup (S \times\{m\} )$}
    \EndFor
    \State{$\hat X \gets $\ScaledPCA($\P_{\Omega}(Y_t), k^{(1)} + k^{(2)}, N$)}
    \LineComment{Least squares iteration}
    \State{$L_1 \gets C^\med \lceil \log M_1 \rceil$}
    \For{$i = 1,\ldots, s_1$}
    \State{ $m \gets M^{\init}  + (i-1)L_1M_1 + 1$}
    \State{ $I \gets [m: (m +  L_1 M_1 -1)]$}
    \State{$\Omega^{(1)}, \Omega^{(2)} \gets $\textsc{Sample}($\hat X, k^{(1)},k^{(2)}, I,a$)}
    \State{$\hat X \gets$ \textsc{MedianLS} ($\hat X,Y_t, \Omega^{(1)}, \Omega^{(2)}$, $M_1$, $m$, $\epsilon$)}
    \State{$\Omega \gets \Omega \cup \Omega^{(1)} \cup \Omega^{(2)}$}
    \EndFor
    \State{$L_2 \gets C^\med \lceil \log M_2 \rceil$}
    \For{$i = 1,\ldots, s_2$}
    \State{ $m \gets M^{\init} + s_1 L_1 M_1 + (i-1)L_2M_2 + 1$}
    \State{ $I \gets [m: (m +  L_2 M_2 -1)]$}
    \State{$\Omega^{(1)}, \Omega^{(2)} \gets $\textsc{Sample}($\hat X, k^{(1)},k^{(2)}, I,a$)}
    \State{$\hat X \gets$ \textsc{MedianLS} ($\hat X,Y_t, \Omega^{(1)}, \Omega^{(2)}$, $M_2$, $m$, $\epsilon$)}
    \State{$\Omega \gets \Omega \cup \Omega^{(1)} \cup \Omega^{(2)}$}
    \EndFor
    \State \Return{$\hat X, \Omega$}
    \EndFunction
  \end{algorithmic}
\end{algorithm}

\begin{theorem}[Random sampling, for small $\sigma_Z/\epsilon$]
  \label{thm:noisy_random_large_epsilon}
  Suppose that $U$, the orthonormal part of $\QR(\X)$, is $k^{(1)}$-isomeric. Suppose further that Assumptions
  \ref{assumption:noise}, \ref{assumption:noise_var_random},   \ref{assumption:subgaussian},  \ref{assumption:singular_vals}, \ref{assumption:incoherence} hold, and $N/2 \geq k^{(1)} \geq r,k^{(2)} \geq 1$, $1 \geq\epsilon $,  and
  Equation \eqref{eqn:epsilon_large} hold.
  Then there exists constants $C_{\ref{thm:noisy_random_large_epsilon}}^\init$, $C_{\ref{thm:noisy_random_large_epsilon}}^\iter$, $C_{\ref{thm:noisy_random_large_epsilon}}^\prob$ such that for all $\epsilon > 0$, if we initialize with $M_\init$ columns, where
  \begin{align*}
    {\textstyle
    M_\init \geq C_{\ref{thm:noisy_random_large_epsilon}}^{\init}
    \frac{\sigma_1(\X)^6 N^2 (\log M_{\init})^3r^2 }{\sigma_r(\X)^6 (k^{(1)} + k^{(2)})^2  \sigma_*(U;k^{(1)})^2}
    , }
  \end{align*}
  and we use $s$ blocks, where
  $\mathrm{s} \geq \log\left( \frac{\sigma_r(\X)\sigma_*(U;k^{(1)})}{48 \sqrt{r}\epsilon}\right),$
  and each block has size $M$, with
  \begin{align*}
    {\textstyle    M \geq  C_{\ref{thm:noisy_random_large_epsilon}}^\iter
    \frac{\sigma_1(\X)^6 r^3 N (\log M)^2}{\sigma_r(\X)^6 k^{(2)} \sigma_*(U;k^{(1)})^2} + \log\left(\frac{1}{\epsilon}\right),
    }
  \end{align*}
  and
  \begin{align}
    \epsilon \geq e^{- C^{\mathrm Q }M} M,
  \end{align}
  then  \ColSpaceEst($Y_t$,  $k^{(1)}, k^{(2)}$,
  $M_\init$,
  $M$,
  $s$,
  $\epsilon$,
  False) returns an $\hat X$ such that
  $\sin \theta (U,\hat X) \leq   \epsilon$
  with probability at least $1-2 M_\init^{-2}-C_{\ref{thm:noisy_random_large_epsilon}}^\prob s M^{-2}$.
\end{theorem}

\begin{theorem}[Random sampling, for large $\frac{\sigma_Z}{\epsilon}$]
  \label{thm:noisy_random_small_epsilon}
  Suppose Assumptions \ref{assumption:noise}, \ref{assumption:noise_var}, \ref{assumption:subgaussian}, \ref{assumption:singular_vals}, \ref{assumption:incoherence} hold.
  Then there exist constants $C_{\ref{thm:noisy_random_small_epsilon}}^{\init}, C_{\ref{thm:noisy_random_small_epsilon}}^{\iter},
    C_{\ref{thm:noisy_random_small_epsilon}}^\prob, C^{(Q)}$ such that for $N/2 \geq k^{(1)} \geq r,k^{(2)} \geq 1$, $1  \geq\epsilon > e^{-C^{(Q)} \min(M_1,M_2)}M  $ ,and  $\epsilon$ satisfying equation \eqref{eqn:epsilon_small},
  if we initialize with $M_\init$ columns, where
  \[
    {\textstyle
    M_\init \geq C_{\ref{thm:noisy_random_small_epsilon}}^\init
    \frac{\sigma_1(\X)^6 N^2 (\log M_{\init})^3r^2 }{\sigma_r(\X)^6 (k^{(1)} + k^{(2)})^2  \sigma_*(U;k^{(1)})^2}
    }
  \]
  and perform alternating minimization with
  $s_1 = \log \left( \frac{\sigma_r(\X) \sigma_*(U;k^{(1)})}{48 \sigma_Z \sqrt{k^{(1)}}}\right) $
  blocks of size
  \begin{align*}
    {\textstyle    M_1 \geq  C_{\ref{thm:noisy_random_large_epsilon}}^\iter
    \frac{\sigma_1(\X)^6 r^3 N (\log M)^2}{\sigma_r(\X)^6 k^{(2)} \sigma_*(U;k^{(1)})^2} + \log\left(\frac{1}{\epsilon}\right)
    },
  \end{align*}
  followed by alternating minimization with $s_2 = 1$ block of size
  \begin{align*}
    {\textstyle
    M_2 \geq C_{\ref{thm:noisy_random_small_epsilon}}^{\iter}
    \max \{\frac{r^2\sigma_Z^2 \sigma_1(\X)^4  N k^{(1)} (\log M)^2}{\sigma_r(\X)^6 k^{(2)}\sigma_*(U;k^{(1)})^2\epsilon^2}, \frac{r\sigma_1(\X)^2 \sqrt{N} \log M}{\sigma_r(\X)^2\sqrt{k^{(2)}}}  \} + \log\left(\frac{1}{\epsilon}\right)
    ,
    }
  \end{align*}
  then \textsc{DoubleColumnSpaceEstimate}($Y_t$, $k^{(1)}$, $k^{(2)}$, $M_\init$, $M_1$, $M_2$, $s_1,s_2, \epsilon$, False)
  returns an $\hat X$ such that
  $\sin \theta (U,\hat X) \leq   \epsilon$
  with probability at least $1 - 2 M_\init^{-2}-C_{\ref{thm:noisy_random_large_epsilon}}^\prob s_1 M_1^{-2} -
    C^\prob_{\ref{thm:noisy_random_small_epsilon}}M_2^{-2} $.
\end{theorem}

\section{Main Proofs}
\label{section:thm_proofs}

The proofs are presented in the following sequence: first, we state general results about noisy subspace iteration  in Section \ref{section:noisy_power} which we will use to prove our theorems. Most of the lemmas that are used in the proofs of the theorems can be found in Section \ref{section:helpful_lemmas}. We defer the proofs of theorems for noisy subspace iteration ( \cite{hardt}) to Section \ref{section:deferred}, and well-established concentration inequalities and random matrix theory results that we need to Section \ref{section:concentration}. Readers who are familiar with techniques from \cite{hardt} need not look at Section \ref{section:deferred}.

Our proof uses noisy subspace iteration for matrix completion, a technique that originated in \cite{jain2013low}, and was expanded upon in \cite{hardt}.

\subsection{Noisy Subspace Iteration}
\label{section:noisy_power}
Noisy subspace iteration generalizes the concept of power iteration, where the top eigenvectors are found by iteratively multiplying a vector by a matrix and then normalizing. In our case, noise is added before the normalization step, and by controlling these noise terms (with the help of Lemmas from Section \ref{section:helpful_lemmas}), we can show convergence to the correct subspace.

Since our problem is asymmetric in the two dimensions $N$ and $M$, we have two (related, but different) lemmas for each least squares step in Algorithm \ref{alg:ls_iteration}. Lemma \ref{lemma:noisy_subspace_W} corresponds to the result at line 4 of Algorithm \ref{alg:ls_iteration} and Lemma \ref{lemma:noisy_subspace} corresponds to the result at Line 6 of Algorithm \ref{alg:ls_iteration}. The proofs of Lemma \ref{lemma:noisy_subspace_W} and Lemma \ref{lemma:noisy_subspace}  are essentially the same as in \cite{hardt} and, therefore, we defer their proofs to Section \ref{section:deferred}.

\paragraph{Notation}  We use $A = U \Sigma V^T$ the singular  value decomposition up to rank $r$, with $U\p \Sigma\p V\p^T$ the completion of orthonormal basis  for $U$ and $V$, respectively. This means that $U \in \R^{N \times r}, \Sigma \in \R^{r \times r}, V \in \R^{M \times r}, U\p \in \R^{N \times (N-r)},  \Sigma\p \in \R^{(N-r) \times (N-r)}, V\p \in \R^{M \times (N-r)}$.  (Note that our notation differs from the notation in \cite{hardt}.)

\begin{lemma}[Noisy Subspace Iteration for $\hatW$]
  \label{lemma:noisy_subspace_W}
  Suppose $A \in \R^{N \times M}$ has rank $r$, and $\hatW$ is the left matrix from the QR decomposition of $\tilde W = A^T X + G^{(1)}$, $\cos \theta(U, X) \geq \frac{1}{2}$, and $\sigma_r(A)\sigma_r(U^TX) > \|V^TG^{(1)}\|$.
  Then
  \begin{align}
    \tan \theta (\hatW, V) \leq \frac{ 2\| V\p^T G^{(1)}\|}{\sigma_r(A) -  2 \|V^T G^{(1)}\|}    \label{eqn:noisy_subspace_W}.
  \end{align}
\end{lemma}

\begin{lemma}[Noisy Subspace Iteration for $X$]
  \label{lemma:noisy_subspace}
  Suppose $A$ has rank $r$, $R \in \R^{r \times r}$ is an invertible matrix,  $X$ is the left  matrix from the $QR$ decomposition of
  \begin{align*}
    \tilde X & = AA^T X^\prev R^{-1} + AG^{(1)}R^{-1} + G^{(2)} ,
  \end{align*}
  and $\cos \theta(U, X^\prev) \geq \frac{1}{2}$ and $\sigma_r(\Sigma)\sigma_r(\Sigma U^T X^\prev  + V^T G^{(1)}) \geq \| U^TG^{(2)}R\| .$
  Then
  \begin{align}
    \tan \theta (U, X) \leq \frac{ 2\| U\p^T G^{(2)}R\|}{\sigma_r(\Sigma) (\sigma_r(\Sigma)- 2\| V^T G^{(1)}\|) -2 \|U^T G^{(2)}R\|}. \label{eqn:noisy_subspace}
  \end{align}
\end{lemma}

In our case, $R$ is the matrix from the $QR$ factorization of $\tilde W_0$ plus noise from performing SmoothQR (Section \ref{subsection:smoothQR}). Note that for  $J \in \R^{M \times r}$, the QR factorization $J = QR$ results in $Q \in \R^{M \times r}, R \in \R^{r \times r}$ where the columns of $Q$ are orthonormal. If $J$ has rank $r$, then $R$ is invertible, and furthermore $\|J\| = \|R\|$.

\subsection{Least squares to Noisy subspace iteration}
Next, Lemma \ref{lemma:G_expression} that says performing least squares is the same as performing noisy subspace iteration. Note that our $A_M = \X W_M$'s are changing with each block (where we split $O(t)$ columns into blocks of size $M$). Letting $A_M = U_M \Sigma_M V_M^T$ be the singular decomposition, it is not true that $U_M = U_{M'}$.  However, the left singular subspace remains the same (because they are both the same as the left singular subspace of $\X$) and this is what matters; for matrices $U, U' \in \R^{N \times r}$  with orthogonal columns that span the same subspace, $UU^T = U'U'^T$, and so $\sin \theta (U, X) = \|(I_N - UU^T) X \| = \|(I_N - U' (U')^T)X \| = \sin \theta (U',  X)$. Therefore, we can assume that $U$ remains fixed. $W_M$ is changing as well, but we do not require that it remains the same.

The proof of Lemma \ref{lemma:G_expression} is in Section \ref{section:deferred}
\begin{lemma}
  \label{lemma:G_expression}
  After Line \ref{alg:second_ls} in Subroutine \ref{alg:ls_iteration},
  we have $\tilde W =  A^TX^\prev + G^{(1)}$ and
  $\tilde X^{(i)} = AA^T X^{\prev}R^{-1} + A G^{(1)}R^{-1} + G^{(2, i)}$, where  $G^{(1)} = G^{(1)}_{A} + G^{(1)}_Z + G^{(1)}_H$, $G^{(2,i)} = G^{(2,i)}_A + G^{(2,i)}_Z$, and $R$ is the $R$-matrix from the QR decomposition of $\tilde W$, which satisfy
  \begin{align*}
    e_m^T G^{(1)}_{A}   & = e_m^T V\Sigma (((I_N- X^\prev (X^\prev)^T)U)^T P_m^{(1)} X^\prev) (B_m^{(1)})^{-1} \\
    e_m^T G^{(1)}_Z     & = e_m^T Z^T P_m^{(1)} X^{\prev} (B_m^{(1)})^{-1}                                     \\
    e_n^T G^{(2,i)}_{A} & = e_n^T U\Sigma (((I_M- \hatW (\hatW)^T)V)^T P_n^{(2,i)} \hatW) (B_n^{(2,i)})^{-1}   \\
    e_n^T G^{(2,i)}_Z   & = e_n^T Z P_n^{(2,i)} \hatW(B_n^{(2,i)})^{-1} .
  \end{align*}
  and $G_H^{(1)}$ is the error resulting from the SmoothQR step as described in Algorithm \ref{alg:SmoothQR}.
\end{lemma}

The notation is as follows: $p^{(1)} = \frac{k^{(1)}}{N}$, $p^{(2)} = \frac{k^{(2)}}{N}$,  with $k^{(j)} = |\Omega^{(j)}|$. For $m \in [M]$ and $n \in [N]$, we define
\begin{align*}
  P_{m}^{(1)}   & : \R^N \to \R^N , \
  P_{m}^{(1)} =  \sum_{l \in [N]: (l,m) \in \Omega^{(1)}} e_l e_l^T,         \\
  P_{n}^{(2,i)} & : \R^M \to \R^M,  \
  P_{n}^{(2,i)} =  \sum_{j \in [M]: (n,j) \in \Omega_{J_i}^{(2)}} e_j e_j^T, \\
  B_{m}^{(1)}   & : \R^r \to \R^r , \
  B_{m}^{(1)} = (X^{\prev})^T P_m^{(1)} X^{\prev}                            \\
  B_{n}^{(2,i)} & : \R^r \to \R^r , \
  B_{n}^{(2,i)} = \hatW^T P_n^{(2,i)} \hatW
\end{align*}

Bounds on $G^{(1)}, G^{(2,i)}$ are shown in Section \ref{section:helpful_lemmas}. After using the elementwise median to combine in Line \ref{alg:line_median} in Algorithm \ref{alg:ls_iteration}, we have the following bounds that are also proved in Section \ref{section:helpful_lemmas}.

\begin{lemma}
  \label{lemma:G_expression_median}
  After Line \ref{alg:line_median} in Subroutine \ref{alg:ls_iteration},
  we have $\tilde W =  A^TX^\prev + G^{(1)}$ and
  $X = AA^T X^{\prev}R^{-1} + A G^{(1)}R^{-1} + G^{(2)}$, where with probability at least $1-O(M^{-2})$
  \begin{enumerate}
    \item  $\|G^{(1)}\| \leq  \max\{\frac{9}{4} \max_{m \in [M]}\|P_m^{(1)} X^\prev (B_m^{(1)})^{-1} \| \sqrt{M}\left(\sin \theta(U, X^\prev)  \sigma_1(\mathring{X}) \sqrt{r} +  \sigma_Z \sqrt{k^{(1)}} \right),  3 \sigma_r(\X) \epsilon \}$.
    \item  $\mu(\hat W) \leq  C' \log M$.
    \item $\| G^{(2)}\| \leq  4 \left(\frac{12 \sqrt{5} r \sigma_1(\X)\sin \theta(V,\hatW) \sqrt{\mu (\hatW) \log M  }}{\sqrt{p^{(2)}}} +   \frac{4\sqrt{5} \sigma_Z\sqrt{r N \log M \mu(\hatW)}}{\sqrt{p^{(2)}}}\right)$.
    \item $\|R\| \leq 3/2 \sigma_1(\X) \sqrt{M} + \|G^{(1)}\|$.
  \end{enumerate}
\end{lemma}

For simplicity, let's first consider the case where there is no noise, $\sigma_Z = 0$. If $\sin \theta (U, X^\prev)\|P_m^{(1)} X^\prev (B_m^{(1)})^{-1}\|$ were sufficiently small, the denominator of Equation \eqref{eqn:noisy_subspace_W} would be bounded from below by $O(\sigma_r(\X)\sqrt{M})$, and we would be able to apply Equation \eqref{eqn:noisy_subspace_W}.

We show this by bounding each factor. So, we require $\sin \theta(U,X^\prev)$ to be small, i.e., we start with some `close' initialization.

Next, we need $\max_{m \in [M]}\|P_m^{(1)} X^\prev (B_m^{(1)})^{-1}\|$ to also be small. This can happen in two ways: one is that we can actively choose $\Omega^{(1)}$, which by Lemma \ref{lemma:best_omega} we can choose to be $\frac{1}{\sqrt{q^{(1)}}}$. If we choose randomly, because $M$ is growing, we require $\|P_m^{(1)} X^\prev (B_m^{(1)})^{-1}\|$ to be bounded for any choice of $\Omega^{(1)}$, otherwise, the probability that this quantity is very  large could get very big as $M$ grows. Therefore, we assume that the true column space $U$ has no bad subset $\Omega^{(1)}$ of size $k^{(1)}$ (i.e., $U$ is $k^{(1)}$-isomeric), and we initialize so close to $U$, so that the minimum value of $\|P_m^{(1)} X^\prev (B_m^{(1)})^{-1}\|$ is also bounded from above, using Lemma \ref{lemma:minorize}
; in this case we can let $\Omega^{(1)}$ to be arbitrary.

In either case, we need a lemma that tells us that our initialization, \ScaledPCA, can get us reasonably close to $U$.

\subsection{Initialization}

\begin{lemma}[Initialization]
  \label{lemma:initialization_general}
  Suppose Assumptions \ref{assumption:singular_vals}, \ref{assumption:noise}, \ref{assumption:subgaussian} holds,
  and let $\Omega \subset [N] \times [M_\init]$ be a subset such that for each $m \in [M_\init]$, $k$ elements $\{(n_i,m)\}_{i=1}^k$ are chosen uniformly at random  to be in $\Omega$. There exists a constant $C_\init > 1$ such that for $\epsilon < 1$, if $M_\init \geq C_\sv$ satisfies
  \[
    M_\init \geq C_\init \frac{\sigma_1(\X)^4 N^2 r (\log M_{\init})^3}{\sigma_r(\X)^4 k^2 \epsilon^2},
  \]
  then \ScaledPCA$(\P_{\Omega}(Y_{M_\init}), \frac{k}{N})$ returns an $\hat X$ that satisfies
  \[
    \sin \theta(\hat X, U) \leq \epsilon
  \]
  with probability at least $1-2 M_\init^{-2}$.
\end{lemma}
\begin{proof}
  We use Proposition 3 of \cite{lounici2014high}
  with their $\delta = p = \frac{k}{N}$, and their $t$ equal to our  $2\log {M_\init}$, which will dominate their $\log(2N)$ term.
  Their $X_t$ corresponds to our $\X w_t+ z_t$, which are not identically distributed, but their proof does not require them to be identically distributed, as long as their Assumption 1 is satisfied.
  Let $C_\scaled$ be the elementwise rescaled empirical covariance matrix as in Line \ref{alg:line_scale} of Algorithm \ref{alg:scaledPCA}. Technically, they use Bernoulli sampling for each entry of each column, but their proof techniques also hold for choosing a subset of size $k$ uniformly at random (up to a constant) for each column (but independently for every column):
  their Lemma 2 holds because they use a union bound over each diagonal entry, so independence need not hold for these sampling events; their Lemma 3 uses the matrix Bernstein inequality for matrices, where the matrices being summed are over each (partially observed) column.

  Note that  $ C_\scaled/{M_\init}$ is an unbiased estimator of $\frac{1}{{M_\init}}(\X W_{M_\init} W_{M_\init} \X^T + \sigma_Z^2 I_N)$. By Proposition 3 of \cite{lounici2014high}, for ${M_\init} \geq 2N$, and ${M_\init}N \geq O(k^2\log {M_\init})$, with their $t$ as our $2\log {M_\init}$,
  \begin{align*}
    \|\frac{1}{{M_\init}} & \left( C_\scaled -  (\X W_{M_\init} W_{M_\init}^T \X^T + \sigma_Z^2 I_N)\right)\|                                                                                                                                          \\
                          & \leq O \left(\frac{1}{{M_\init}}\|\X W_{M_\init} W_{M_\init}^T \X^T + \sigma_Z^2 I_N\|\max \left\{\sqrt{\frac{3 N^2 r\log {M_\init} }{k^2 {M_\init}}},  \frac{9N^2 r (\log {M_\init})^2}{k^2 {M_\init}}  \right\} \right).
  \end{align*}
  For $k^2 {M_\init} \geq O( N^2 r (\log {M_\init})^3)$ (which holds by the lemma assumption), the first term is larger, so
  \begin{align*}
    \|\frac{1}{{M_\init}}\left( C_\scaled -  (\X W_{M_\init} W_{M_\init}^T \X^T + \sigma_Z^2 I_N)\right)\|
     & \leq O \left(\sigma_1(\X)^2\sqrt{\frac{ N^2 r\log {M_\init} }{k^2 {M_\init}}} \right).
  \end{align*}

  Once we have that these matrices are close, we can deduce that their singular vectors are close by Wedin's Theorem (Theorem \ref{thm:wedin}). First, note that adding a scalar multiple of $I_N$ does not change eigenspaces, and it shifts the eigenvalues. Then, we have
  \begin{align*}
    \sin \theta (U, \hat X) & \leq O \left(\sigma_1(\X)^2 \sqrt{\frac{ N^2 r\log {M_\init} }{k^2 {M_\init} }}\frac{1}{\sigma_r(\X W_t W_t^T \X^T + \sigma_Z^2 I_N)} \right) \\
                            & \leq O \left(\frac{\sigma_1(\X)^2 \sqrt{N^2 r \log {M_\init} }}{\sqrt{k^2 M_\init} \sigma_r(\X)^2} \right)
  \end{align*}
  The conclusion of the lemma follows from this inequality.
\end{proof}

For active sampling, we do not need to get as good of an initial guess at $U$. By Lemma \ref{lemma:best_omega},  we can always choose $\Omega^{(1)}$ so that $\|P_m^{(1)} X^\prev (B_m^{(1)})^{-1}\| \leq  \frac{1}{\sqrt{q^{(1)}}}$.

\begin{lemma}[Initialization for Active Sampling]
  \label{lemma:init}
  For every $C' > 0$, there exists a $C$ such that for``'
  \begin{align}
    M_\init \geq C \frac{\sigma_1(\X)^6 N^2 (\log M_{\init})^3r^2 }{\sigma_r(\X)^6 (k^{(1)} + k^{(2)})^2  q^{(1)}}, \label{eqn:active_init}
  \end{align}
  $\ScaledPCA(\P_{\Omega_{M_\init}}(Y_t))$ returns a matrix,   which we denote as  $\hat X(M_\init)$,
  such that
  \[
    \sin \theta (U, \hat X(M_\init)) \leq C'\frac{\sigma_r(\mathring{X}) \sqrt{q^{(1)}}}{\sigma_1(\mathring{X})\sqrt{r}}
  \]
  with probability at least $1-2M_\init^{-2}$.
\end{lemma}
\begin{proof}
  This follows from Lemma \ref{lemma:initialization_general}.
\end{proof}

Using Lemma \ref{lemma:init}, with active sampling, with $M_\init$ as in Equation \eqref{eqn:active_init},  with probability at least $1-2M_\init^{-2}$, we have, for the first iteration.
\begin{align}
  \|G_A^{(1)}\| \leq \frac{3C'}{2}\sigma_r(\mathring{X}) \sqrt{M}\label{eqn:GA_t}.
\end{align}

\subsection{Lower bounding the denominators}
\begin{lemma}
  \label{lemma:denom_bound}
  With probability $1-2M_\init^{-2}-O(sM^{-2})$, for every iteration over blocks, the denominator of Equation \eqref{eqn:noisy_subspace_W} is bounded from below  by
  \begin{align}
    \frac{9 \sigma_r(\X)\sqrt{M}}{16} \label{eqn:denom1}
  \end{align}
  and  the denominator of Equation \eqref{eqn:noisy_subspace} is bounded from below  by
  \begin{align}
    \frac{1}{4} \sigma_r(\X)^2 M \label{eqn:denom2}.
  \end{align}
\end{lemma}
\begin{proof}
  The initialization step uses Lemma \ref{lemma:init} with $C' = \frac{1}{48}$. Combining this with Assumption \ref{assumption:noise_var} and \ref{lemma:G_expression_median} and for $M$ larger than a constant, we have
  \begin{align*}
    \|G^{(1)}\| \leq \frac{9}{4} \sqrt{\frac{M}{q^{(1)}}}\frac{2\sqrt{q^{(1)}} \sigma_r(\X)}{48} = \frac{3}{32} \sigma_r(\X)\sqrt{M}
  \end{align*}
  This means the denominator of Equation \eqref{eqn:noisy_subspace_W} is bounded from below
  \[
    \sigma_r(A) - 2 \|G^{(1)}\| \geq \frac{3}{4}\sigma_r(\X) \sqrt{M} - \frac{3 \sigma_r(\X) \sqrt{M}}{16} = \frac{9 \sigma_r(\X)\sqrt{M}}{16}
  \]

  To bound the denominator in Equation \eqref{eqn:noisy_subspace_W}, we bound the following:
  \[
    \sigma_r(\Sigma)(\sigma_r(\Sigma) - 2 \|G^{(1)}\|) \geq \frac{3}{4} \sigma_r(\X) \sqrt{M}\frac{9\sigma_r(\X)\sqrt{M}}{16} = \frac{27}{64} \sigma_r(\X)^2 M,
  \]
  and
  \begin{align*}
    \| G^{(2)}\| & \leq   O \left(\frac{\log M(r\sigma_1(\X)  + \sigma_r(\X))}{\sqrt{p^{(2)}}}  \right)
    = O \left(\frac{\sigma_1(\X) (\log M) \sqrt{N}}{\sqrt{k^{(1)}}}  \right),
  \end{align*}
  and
  \begin{align*}
    \|R\| \leq \|A^T X^\prev + G^{(1)}\| \leq \|A \| + \|G^{(1)}\| \leq \frac{3}{2} \sigma_1(\X) \sqrt{M}  + \frac{3}{32} \sigma_r(\X) \sqrt{M} \leq 2 \sigma_1(\X)  \sqrt{M}.
  \end{align*}
  We can bound the denominator from below by
  \[
    \sigma_r(\Sigma)(\sigma_r(\Sigma) - 2 \|G^{(1)}\|)  - 2 \|G^{(2)}\| \|R \|  \geq \frac{27}{64} \sigma_r(\X)^2 M -  C \left(\frac{\sigma_1(\X)^2 r(\log M) \sqrt{NM}}{\sqrt{k^{(2)}}}  \right).
  \]
  for some constant $C$. Therefore, for
  \[
    M \geq O \left(\frac{\sigma_1(\X)^2 r^2 (\log M)^2 N}{\sigma_r(\X)^2 k^{(2)}} \right)
  \]
  the denominator of Equation \eqref{eqn:noisy_subspace} is bounded from below by $\frac{1}{4} \sigma_r(\X)^2 M$.
\end{proof}

Combining the results from  Lemmas \ref{lemma:denom_bound},\ref{lemma:noisy_subspace}, \ref{lemma:noisy_subspace_W},\ref{lemma:best_omega} , \ref{lemma:G_expression_median},
\begin{lemma}
  \label{lemma:master}
  We have
  \begin{align}
    \sin\theta(\hat W, V)
    \leq \max \left\{ O \left(\frac{\sin \theta (U, X^\prev) \sigma_1(\X) \sqrt{r} + \sigma_Z \sqrt{k^{(1)}}}{\sigma_r(\X) \sqrt{q^{(1)}}} \right) , O\left(\frac{ \epsilon }{\sqrt{M}}\right)\right\}\label{eqn:VW_master}
  \end{align}
  and
  \begin{align}
    \sin\theta(U, X)
     & \leq  O \left(\frac{\sigma_1(\X)\sqrt{N}(r\sigma_1(\X) \sin \theta( V, \hatW) \log M + \sigma_Z \sqrt{rN} \log M)}{\sigma_r(\X)^2\sqrt{Mk^{(2)}}} \right) \label{eqn:UX_master}
  \end{align}
\end{lemma}

\subsection{Proof of Main Results}
Now, we are ready to prove the main results of our paper.

\begin{proof}[Proof of Theorem \ref{thm:noisy_active_large_epsilon}]
  We claim  that after alternating minimization with $s$ blocks, where
  \begin{align*}
    M \geq \max \left\{ O\left( \frac{\sigma_1(\X)^4 (\log M)^2 r^2 N^2}{\sigma_r(\X)^4k^{(1)}k^{(2)}}\right),O\left(\frac{ r^3 \sigma_1(\X)^6 (\log M)^2 N}{\sigma_r(\X)^6 k^{(2)} q^{(1)}}\right)
    \right\},
  \end{align*}
  we have
  \begin{align}
    \sin \theta (U, X) \leq
    \max\left \{\epsilon, 2^{-s} \frac{\sigma_r(\X) \sqrt{q^{(1)}}}{48\sigma_1(\X)\sqrt{r}} \right \}
  \end{align}
  We prove this by induction.

  First, we bound $\sin \theta(\hat W, V)$. Using Equation \eqref{eqn:VW_master} from Lemma \ref{lemma:master},
  \begin{align*}
    \sin\theta(\hat W, V)     \leq \max \left \{ O \left(\frac{\sin \theta (U, X^\prev) \sigma_1(\X) \sqrt{r} + \sigma_Z \sqrt{k^{(1)}}}{\sigma_r(\X) \sqrt{q^{(1)}}} \right) , O\left(\frac{ \epsilon}{ \sqrt{M}} \right)\right \}
  \end{align*}
  Using Equation \eqref{eqn:epsilon_large} (and noting that for $M$ satisfying Equation \eqref{eqn:Mlim1}, the second term in the maximum is smaller than the term containing $\epsilon$ in the following equation) we have
  \begin{align}
    \sin\theta(\hat W, V)     \leq O \left(\frac{\sin \theta (U, X^\prev) \sigma_1(\X) \sqrt{r} + \epsilon \sigma_1(\X) \sqrt{r}}{\sigma_r(\X) \sqrt{q^{(1)}}} \right)  = O\left(\frac{\sigma_1(\X) \sqrt{r}}{\sigma_r(\X) \sqrt{q^{(1)}}}(\sin \theta(U, X^\prev) + \epsilon) \right) \label{eqn:le_WV}
  \end{align}

  We have, using Equation \eqref{eqn:UX_master} from Lemma \ref{lemma:master},
  \begin{align*}
    \sin\theta(U, X)
     & \leq  O \left(\underbrace{\frac{r\sigma_1(\X)^2 \sqrt{N}\sin \theta( V, \hatW) \log M}{\sigma_r(\X)^2\sqrt{Mk^{(2)}}} }_{\text{Quantity} A} + \underbrace{\frac{ \sigma_1(\X)\sigma_Z N\sqrt{r} \log M}{\sigma_r(\X)^2\sqrt{Mk^{(2)}}}}_{\text{Quantity} B} \right)
  \end{align*}

  To bound Quantity $A$, we substitute in Equation \eqref{eqn:le_WV}  into the above to get
  \begin{align*}
    C'' \frac{\sigma_1(\X)^3r\sqrt{rN} \log M}{\sigma_r(\X)^3\sqrt{Mk^{(2)}q^{(1)}}} (\sin \theta (U, X^\prev) + \epsilon )  .
  \end{align*}
  So for
  \begin{align}
    M \geq  36 C'' \frac{\sigma_1(\X)^6 r^3 N (\log M)^2}{\sigma_r(\X)^6 k^{(2)} q^{(1)}}, \label{eqn:Mlim1}
  \end{align}
  Quantity $A$ is bounded by
  \[
    \frac{\sin \theta(U, X^\prev)}{6} + \frac{\epsilon}{6},
  \]
  and we claim that each quantity is bounded by  $\frac{\max\{\epsilon, 2^{-s} \frac{\sigma_r(\X) \sqrt{q^{(1)}}}{48\sigma_1(\X)\sqrt{r}}  \}}{3}$.\\

  By induction, $ \sin \theta(U, X^\prev) \leq \max\{\epsilon, 2^{-s} \frac{\sigma_r(\X) \sqrt{q^{(1)}}}{48\sigma_1(\X)\sqrt{r}}  \}$. Therefore,
  \[
    \frac{1}{6} \sin \theta(U, X^\prev)  \leq  \max\{1/6 \epsilon, 3^{-1} 2^{-s-1} \frac{\sigma_r(\X) \sqrt{q^{(1)}}}{48\sigma_1(\X)\sqrt{r}}  \}.
  \]
  For $\frac{1}{6} \epsilon \leq  1/3 \max\{\epsilon, 2^{-s} \frac{\sigma_r(\X) \sqrt{q^{(1)}}}{48\sigma_1(\X)\sqrt{r}}  \}$: if the maximum $\epsilon$, this is clear. Otherwise, suppose $\max\{\epsilon, 2^{-s} \frac{\sigma_r(\X) \sqrt{q^{(1)}}}{48\sigma_1(\X)\sqrt{r}}\} = 2^{-s} \frac{\sigma_r(\X) \sqrt{q^{(1)}}}{48\sigma_1(\X)\sqrt{r}}$, i.e., $\epsilon \leq 2^{-s} \frac{\sigma_r(\X) \sqrt{q^{(1)}}}{48\sigma_1(\X)\sqrt{r}}$. Then, this quantity is bounded by
  \begin{align*}
    \frac{\epsilon}{3} \leq 6^{-1} 2^{-s}\frac{\sigma_r(\X) \sqrt{q^{(1)}}}{48\sigma_1(\X)\sqrt{r}} \leq \frac{1}{3} \left(2^{-s-1} \frac{\sigma_r(\X) \sqrt{q^{(1)}}}{48\sigma_1(\X)\sqrt{r}}\right)
  \end{align*}
  Therefore, Quantity $A$ is bounded by $1 /3 \max\{\epsilon, 2^{-s} \frac{\sigma_r(\X) \sqrt{q^{(1)}}}{48\sigma_1(\X)\sqrt{r}}  \}$ .

  Quantity $B$ is bounded by
  \begin{align*}
    \sin\theta(U, X)
     & \leq  C''' \left(\frac{ \sigma_1(\X)^2 rN \log M}{\sigma_r(\X)^2\sqrt{Mk^{(2)}k^{(1)}}} \right) \epsilon
  \end{align*}
  So for
  \[
    M \geq  36 C''' \frac{\sigma_1(\X)^4 r^2 N^2 (\log M)^2}{\sigma_r(\X)^4 k^{(2)}k^{(1)}} ,
  \]
  Quantity $B$ is bounded by $1/6\epsilon$, which by the above discussion, is bounded by  $1/3 \max\{\epsilon, 2^{-s} \frac{\sigma_r(\X) \sqrt{q^{(1)}}}{48\sigma_1(\X)\sqrt{r}}  \}$.

\end{proof}

\begin{proof}[Proof of Theorem \ref{thm:noisy_active_small_epsilon}]

  We have, as before,  using Equation \eqref{eqn:UX_master} from Lemma \ref{lemma:master},
  \begin{align*}
    \sin\theta(U, X)
     & \leq  O \left(\underbrace{\frac{r\sigma_1(\X)^2 \sqrt{N}\sin \theta( V, \hatW) \log M}{\sigma_r(\X)^2\sqrt{Mk^{(2)}}} }_{\text{Quantity} A} + \underbrace{\frac{ \sigma_1(\X)\sigma_Z N\sqrt{r} \log M}{\sigma_r(\X)^2\sqrt{Mk^{(2)}}}}_{\text{Quantity} B} \right)
  \end{align*}

  Quantity $B$ is bounded by  $\epsilon / 2$ when
  \begin{align*}
    M \geq 4 (C'')^2 \frac{\sigma_1(\X)^2\sigma_Z^2 r N^2 (\log M)^2}{\sigma_r(\X)^4 k^{(2)}\epsilon^2}.
  \end{align*}

  Next, we bound Quantity $A$ by $\epsilon /2$. By Theorem \ref{thm:noisy_active_large_epsilon}, we have $\sin \theta (U, X^\prev) \leq \frac{\sigma_Z \sqrt{k^{(1)}}}{\sigma_1(\X)\sqrt{r}}$. And using Equation \eqref{eqn:VW_master}
  \begin{align*}
    \sin \theta(\hat W, V) \leq   C \max \left \{   \frac{\sigma_Z \sqrt{k^{(1)}} }{\sigma_r(\X) \sqrt{q^{(1)}}}, \frac{\epsilon}{\sqrt{M}}\right\}.
  \end{align*}
  Therefore, Quantity A is bounded by
  \begin{align*}
    C' \max \left \{\frac{r\sigma_Z \sigma_1(\X)^2 \sqrt{Nk^{(1)}} \log M}{\sigma_r(\X)^3\sqrt{Mk^{(2)}q^{(1)}}}, \epsilon \frac{r\sigma_1(\X)^2 \sqrt{N} \log M}{\sigma_r(\X)^2\sqrt{k^{(2)}}M}  \right\},
  \end{align*}
  which is bounded by $\epsilon/2$ when
  \begin{align*}
    M \geq \max \left \{4 (C')^2 \frac{r^2\sigma_Z^2 \sigma_1(\X)^4 N k^{(1)} (\log M)^2}{\sigma_r(\X)^6 k^{(2)} q^{(1)}\epsilon^2}, 2C' \frac{r\sigma_1(\X)^2 \sqrt{N} \log M}{\sigma_r(\X)^2\sqrt{k^{(2)}}}   \right\},
  \end{align*}
  and for this $M$, the second term in the maximum is also bounded by $\epsilon/2$, therefore Quantity A is bounded by $\epsilon/2$.
\end{proof}

\begin{proof}[Proof of Theorem \ref{thm:noisy_random_large_epsilon} and \ref{thm:noisy_random_small_epsilon}]

  By Lemma \ref{lemma:minorize}, if we start within $\sigmaU/2$ of $U$, then $\sigma_r(X^{\prev}_{\Omega,:}) \geq \sigmaU/2$ for all $\Omega$ of size $k^{(1)}$. 
  $\sigma_r(X^{\prev}_{\Omega,:}) \geq \sigmaU/2$.

  In order to proceed as in Theorem \ref{thm:noisy_active_large_epsilon} and \ref{thm:noisy_active_small_epsilon}, we need to have
  \[
    \sin \theta (U, X^\prev) \leq \min\left\{ \frac{\sigmaU}{2}, \frac{1}{48} \frac{\sigma_r(\X) \sigmaU}{\sigma_1(\X) \sqrt{r}}\right\},
  \]
  and we know the second term will dominate.
  We can now proceed exactly the same as for Theorems \ref{thm:noisy_active_large_epsilon} and \ref{thm:noisy_active_small_epsilon}.
\end{proof}

\section{Helpful Lemmas}
\label{section:helpful_lemmas}

\begin{lemma}
  \label{lemma:best_omega}
  There exists  a set $(\Omega^{(1)})_m$ of size $k$ such that $\|P_m^{(1)} X^\prev ( B_m^{(1)})^{-1}\| \leq \sqrt{1 + \frac{r(N-k)}{k-r+1}}$.
\end{lemma}
\begin{proof}
  Note that $\|P_m^{(1)} X^\prev (B_m^{(1)})^{-1}\| = \sigma_r(P_m^{(1)}X^\prev)^{-1} $. By Corollary 1 of \cite{de2007subset}, there exists an $ (\Omega^{(1)})_m$ such that $\sigma_r(P_m^{(1)} X^\prev) \geq \sqrt{\frac{k-r+1}{r(N-k)+k-r+1}}$.

\end{proof}

\begin{proof}[Proof of Lemma \ref{lemma:G_expression_median}]
  Items 1 and 2 follow directly from Lemma \ref{lemma:G1}.

  For Item 4,
  \begin{align*}
    \|R\| \leq \|A^T X^\prev + G^{(1)}\| \leq \|A \| + \|G^{(1)}\| \leq \frac{3}{2} \sigma_1(\X) \sqrt{M}  + \|G^{(1)}\|.
  \end{align*}

  For Item 3, the proof is similar to the proof of Lemma 4.4 in \cite{hardt}. Note that Line \ref{alg:line_median} takes the median, over $i$,  of $AA^T X^{\prev}R^{-1} + A G^{(1)}R^{-1} + G^{(2,i)}$, and only $G^{(2,i)}$ depends on $i$, so this is equivalent to $AA^T X^{\prev}R^{-1} + A G^{(1)}R^{-1} + G^{(2)}$, where   $G^{(2)} = $ median($G^{(2,1)}, \ldots, G^{(2,\lceil C^\med \log M \rceil)} $) is the elementwise median.  Let $\tau = \lceil C^\med \log M \rceil$, and let
  \[
    B =  \frac{12 \sqrt{5} r \sigma_1(\X)\sin \theta(V,\hatW) \sqrt{\mu (\hatW) \log M  }}{\sqrt{p^{(2)}}} +   \frac{4\sqrt{5} \sigma_Z\sqrt{r N \log M \mu(\hatW)}}{\sqrt{p^{(2)}}}.
  \]
  By Lemma \ref{lemma:G2}, $\Pr [\|G^{(2,i)}\| \geq B ] \leq \frac{4}{5} - 3M^{-2}$.
  We claim that
  \[
    \Pr[ \| G^{(2)} \| \geq 3B ] \leq \exp (-\Omega(C^\med \log M)).
  \]
  To show this, let $S = \{ i \in [\tau]: \| G^{(2,i)}\| \leq B\}$.
  We know that $\E[|S|] \geq \frac{3.9}{5}\tau$,  for $M$ large enough,
  and that the draws over $i$ are independent. Therefore, by Chernoff $|S| > \frac{2}{3} \tau$ with probability $1- \exp(-\Omega(\tau))$. We claim that when this event occurs $\|G^{(2)}\|_F \leq 4 B$.
  To prove this claim, fix any coordinate $(n,r') \in [N] \times [r]$. By the median property $|\{i: (G^{(2,i)}_{n,r'})^2 \geq (G^{(2)}_{n,r'})^2\}| \geq \tau/2$. Since $|S| >  \frac{2}{3} \tau$, this means that at least $\tau/6$ $G^{(2,i)}$'s with $i \in S$ have  $(G^{(2,i)}_{n,r'})^2 \geq (G^{(2)}_{n,r'})^2$. Therefore, the average value of $(G^{(2,i)}_{n,r'})^2$ in $S$ is at least $(G^{(2)}_{n,r'})^2/6$. Summing over $(n,r')$, we obtain that the average value of $\|G^{(2,i)}\|_F^2$ is at least $G^{(2)}/6$. On the other hand, we know that the average of $\|G^{(2,i)}\|$ in $S$ is at most $B^2$ by the definition of the set $S$. It follows that $\|G^{(2)}\|_F^2 \leq 6 B^2$. 
  Taking the square root and taking $C^\med$ large enough yields the desired result.
\end{proof}

\begin{lemma}
  \label{lemma:G1}
  With probability at least $1-O(M^{-2})$, $G^{(1)} = G_A^{(1)} + G_Z^{(1)} + G_H^{(1)}$, where
  \begin{align}
    \|G_A^{(1)}\| & \leq \sin \theta(U, X^\prev) \frac{3}{2}\sigma_1(\mathring{X}) \sqrt{Mr}  \max_{m \leq M}\|(P_m^{(1)} X^\prev (B_m^{(1)})^{-1})\| =: b_{G_A^{(1)}}\label{eqn:GA} \\
    \|G_Z^{(1)}\| & \leq  \frac{3}{2}\sigma_Z \sqrt{k^{(1)}M}\max_{m \leq M}\|(P_m^{(1)} X^\prev (B_m^{(1)})^{-1})\| =: b_{G_Z^{(1)}}                                                \\
    \|G_H^{(1)}\| & \leq \max\{\frac{1}{2}(b_{G_A^{(1)}} + b_{G_Z^{(1)}}), \sigma_r(\X)\epsilon\}
  \end{align}
  and moreover, $\mu(\hatW) \leq C' (\mu(V) + \log M )$.
\end{lemma}
\begin{proof}
  By Lemmma \ref{lemma:G_expression},
  \begin{align*}
    \|e_m^T G_A^{(1)}\| & \leq \|e_m^T V \Sigma \| \|((I_N - X^\prev (X^\prev)^T) U)^T \|\|P_m^{(1)} X^\prev (B_m^{(1)})^{-1}\|.
  \end{align*}
  Note that $\|((I_N - X^\prev (X^\prev)^T)U )^T \| = \sin \theta (U,X^\prev)$. By Assumption \ref{assumption:singular_vals}, $\|V\Sigma \| \leq \frac{3}{2} \sigma_1(\mathring{X}) \sqrt{M}$ with probability at least $1-1/M^2$.
  \begin{align*}
    \|    e_m^TG_A^{(1)} \|
     & \leq \sin \theta (U,X^\prev) \|e_m^T V \Sigma \| \|P_m^{(1)} X^\prev (B_m^{(1)})^{-1}\|    \\
     & \leq \sin \theta (U,X^\prev) \|e_m^T V\|\|\Sigma \| \|P_m^{(1)} X^\prev (B_m^{(1)})^{-1}\| \\
  \end{align*}
  This also means that, using incoherence of $V$, with probability $1-O(M^{-2})$, for all $m \in [M]$
  \begin{align}
    \frac{\sqrt{M}\|e_m^T G_A^{(1)}\|}{\sqrt{r}b_{G_A^{(1)}}} \leq O(\sqrt{\log M}),
  \end{align}
  which we will use when bounding $G_H^{(1)}$.

  So with probability at least $1-1/M^2$,
  \begin{align*}
    \|G_A^{(1)}\|_F^2 & = \sum_{m=1}^M \|e_m^T G_A^{(1)}\|^2                                                                                     \\
                      & \leq \sin \theta(U,X^\prev)^2 \max_{m \leq M} \|P_m X^\prev (B_m^{(1)})^{-1}\|^2 \sum_{m=1}^M \|e_m^T V\|^2 \|\Sigma\|^2 \\
                      & \leq \sin \theta(U,X^\prev)^2 \max_{m \leq M} \|P_m X^\prev (B_m^{(1)})^{-1}\|^2  r  \frac{9\sigma_1(\X)^2 M}{4} ,
  \end{align*}
  from which Equation \eqref{eqn:GA} follows.
  To bound the $G_Z^{(1)}$ terms, note that $(P_m^{(1)})^2 = P_m^{(1)}$ , so
  \begin{align*}
    e_m^T G_Z^{(1)} =\left(e_m^T Z^T P_m^{(1)}\right) \left(P_m^{(1)} X^\prev (B_m^{(1)})^{-1}\right)
  \end{align*}
  Since each entry of $Z$ is independent from everything else, this has the same distribution as
  \begin{align*}
    \left(e_m^T Z_k \tilde P_m^{(1)}\right) \left(P_m^{(1)} X^\prev (B_m^{(1)})^{-1}\right)
  \end{align*}
  where $Z_k \in \R^{M \times k^{(1)}}$ has i.i.d. $\N(0,\sigma_Z^2)$ entries, and $\tilde P_m^{(1)} \in \R^{k^{(1)} \times N}$ `projects' back up to $\R^N$ according to $\Omega^{(1)}$
  \begin{align*}
    \|e_m^T G_Z^{(1)}\| & \leq \|e_m^T Z_k \| \|(P_m^{(1)} X^\prev (B_m^{(1)})^{-1})\|
  \end{align*}
  As with before, taking the summation $\sum_{m=1}^M \|e_m^TZP_m^{(1)}\|^2  \overset{\mathrm{distribution}}{=} \|Z_k\|_F^2 \leq k^{(1)}\left(\sigma_Z \frac{3}{2}\sqrt{M}\right)^2$ for $M \geq O((k^{(1)})^2)$ with probability at least $1 -1/M^2$ by Theorem \ref{thm:singular_vals}, so    $\|G_Z^{(1)}\| \leq \frac{3}{2}\sigma_Z \sqrt{k^{(1)}M}\|P_m^{(1)} X^\prev (B_m^{(1)})^{-1}\|$.

  We also note that with probability $1-O(M^{-2})$, for all $m \in [M]$
  \begin{align}
    \frac{\sqrt{M}\|e_m^T G_Z^{(1)}\|}{\sqrt{r}b_{G_Z^{(1)}}} \leq O(\sqrt{\log M}),
  \end{align}
  which we will use to bound $G^{(1)}_H$.

  To bound $\|G_H^{(1)}\|$, we use Lemma \ref{lemma:smoothQR} with  $ G = G_A^{(1)} + G_Z^{(1)}$, $\nu=\max\{b_{G_A^{(1)}} + b_{G_Z^{(1)}}, 2 \sigma_r(\X) \epsilon\}$, and $\tau = \frac{1}{2}$, and $\epsilon'=\epsilon \sigma_r(\X) $ . This gives that $\mu( \hatW) \leq C' (\mu(V) + \log M )$, and $\|G_H^{(1)}\| \leq \max\{\frac{1}{2}(b_{G_A^{(1)}} + b_{G_Z^{(1)}}), \epsilon \sigma_r(\X)\} $\\

\end{proof}
\begin{lemma}
  \label{lemma:G2}
  For the $G^{(2,i)}$'s in Lemma \ref{lemma:G_expression}, we have, with probability at least $\frac{4}{5} - 3M^{-2}$,
  \begin{align}
    \|G^{(2,i)}_A\| & \leq   \frac{12 \sqrt{5}  r \sigma_1(\X) \sin \theta(V,\hatW) \sqrt{ \mu(\hatW)\log M }}{\sqrt{p^{(2)}}} \text{ and} \\
    \|G^{(2,i)}_Z\| & \leq  \frac{4\sqrt{30}\sigma_Z \sqrt{rN  \mu(\hat W) \log M } }{\sqrt{p^{(2)}}}.
  \end{align}
\end{lemma}

\begin{proof}
  $\|G_A^{(2,i)}\|$ can be bounded by $\|G_A^{(2,i)}\|_F$, which we can bound by
  \begin{align}
    \|e_n^T G_A^{(2,i)}\|           & = \|e_n^T U \Sigma ((I_M - \hatW\hatW^T)V)^T P_n^{(2,i)} \hatW) (B_n^{(2,i)})^{-1}\| \nonumber                                                                                          \\
                                    & \leq \|e_n^T U\Sigma ((I_M - \hatW\hatW^T)V)^TP_n^{(2,i)} \hatW\| \|(B_n^{(2,i)})^{-1}\|  \nonumber                                                                                     \\
    \implies    \|G_A^{(2,i)}\|_F^2 & \leq \left(\sum_{n=1}^N\|e_n^T U\Sigma ((I_M - \hatW\hatW^T)V)^TP_n^{(2,i)} \hatW\|^2\right) \max_{n \in [N]}\|(B_n^{(2,i)})^{-1}\|^2                                                   \\
    \implies    \|G_A^{(2,i)}\|_F   & \leq \underbrace{\sqrt{\sum_{n=1}^N\|e_n^T U\Sigma ((I_M - \hatW\hatW^T)V)^TP_n^{(2,i)} \hatW\|^2}}_{=\|\tilde G_A^{(2,i)}\|_F} \max_{n \in [N]}\|(B_n^{(2,i)})^{-1}\|  \label{eqn:GA2}
  \end{align}
  Similarly, we can bound $\|G_Z^{(2,i)}\|$ by $\|G_Z^{(2,i)}\|_F$, which we bound by
  \begin{align}
    \|G_Z^{(2,i)}\|_F & \leq \underbrace{\sqrt{\sum_{n=1}^N\|e_n^T P_n^{(2,i)}Z \hatW\|^2}}_{=\|\tilde G_Z^{(2,i)}\|_F} \max_{n\in [N]}\|(B_n^{(2,i)})^{-1}\|.  \label{eqn:GZ2}
  \end{align}
  We define $\tilde G_A^{(2,i)}$ and $\tilde G_Z^{(2,i)}$ by each row:
  \[
    e_n^T \tilde G_A^{(2,i)} := e_n^T U \Sigma ((I_M - \hatW \hatW^T)V)^T P_n^{(2,i)} \hat W
  \]
  and
  \[
    e_n^T \tilde G_Z^{(2,i)} := e_n^T Z P_n^{(2,i)} \hat W,
  \]
  so that the underbraces in Equations \eqref{eqn:GA2} and \eqref{eqn:GZ2} hold.

  Both of these share a factor of $\max_{n \in [N]}\|(B_n^{(2,i)})^{-1}\|^2$, which we will bound first.  Since $\|(B_n^{(2,i)})^{-1}\| = \lambda_{\min}(B_n^{(2,i)})^{-1}$, we use Matrix Chernoff (Theorem \ref{thm:chernoff}). We are bounding the sum of $M$ variables $\tilde E_m = p_m c_mc_m^T$, where $c_m$ is the $m$-th row of $\hatW$ and $p_m$ is i.i.d. Bernoulli($p^{(2)}$). $B_n^{(2,i)} = \sum_{m=1}^{M} \tilde E_m$. $\lambda_{\min} (\sum_m \E[\tilde E_m]) = \lambda_{\min}(\sum_{m=1}^{M} p^{(2)} c_mc_m^T) = \lambda_{\min}(p^{(2)}\hatW^T\hatW) = \lambda_{\min}(p^{(2)}I_{r})$.
  So $\mu_{\min} =  p^{(2)}$, and $R = \lambda_{\max}(\tilde E_m) \leq \frac{\mu(\hatW)r}{M}$. So, by Matrix Chernoff (Theorem \ref{thm:chernoff}) with $t = 1/2$,
  \[
    \Pr[\lambda_{\min}(B_n^{(2,i)}) \leq \frac{1}{2} p^{(2)}] \leq r \exp \left(\frac{- M }{8\mu(\hatW)r} \right).
  \]

  For $M \geq O((\log M)^3 r \log r)$,  this happens with probability at least $1-1/M^3$, and we union bound over the $N$ rows to hold for all $n$, and if $M \geq N$, this is at least $1-1/M^2$.
  So with probability least $1-1/M^2$,
  \begin{align}
    \max_{n \in [N]}\| (B_n^{(2,i)})^{-1}\| \leq \frac{2}{ p^{(2)}} \label{eqn:B2inv}.
  \end{align}

  Next, we bound vectors of the form  $\|e_n^T F P_n^{(2,i)} \hat W\|$ for $F \in \R^{N \times M}$ with $F \hat W = 0$ or $\E[ F_{ij} ]= 0$ for all $i,j$  independently. By Lemma 7.5 of \cite{hardt}, we can replace $P_n^{(2,i)}\hatW$ with $P_n^{(2,i)}\hatW'$ where $\hatW'$ has columns with $\ell_\infty$ norm at most $\sqrt{8 \mu(\hatW)\log M/M}$. Letting $\hatW'_j \in \R^{M \times 1}$ be the columns of $\hatW'$, we have
  \begin{align}
    \|e_n^T F P_n^{(2,i)} \hatW'\|^2 & = \sum_{j=1}^{r} (e_n^T F  P_n^{(2,i)} \hatW')_j^2 \label{eqn:fw}
  \end{align}
  Let's look at the expectation of a single term in this summation: letting $f_n^T := e_n^T F \in \R^{1 \times M}$,
  \begin{align}
    \E [(e_n^T F P_n^{(2,i)} \hatW')_j^2 ] & =  \E[f_n^T P_n^{(2,i)} \hatW'_j]                                                                                                               \\
                                           & \leq \E[\sum_{m=1}^M ((P_n^{(2,i)})_{mm} (f_n)_m (\hatW_j')_m)^2] = \E[\sum_{m=1}^M (P_n^{(2,i)})_{mm} (\tilde f_m^{(n,j)})^2] \label{eqn:fw_j}
  \end{align}

  where $\tilde f^{(n,j)} \in \R^M$ is the Hadamard (elementwise) product of the vectors $f_n$ and $\hatW'_j$.  Note that $\|\tilde f^{(n,j)}\|^2$ is bounded by
  \begin{align}
    \| \tilde f^{(n,j)}\|^2 = \sum_{m=1}^M (\tilde f^{(n,j)}_m)^2 =  \sum_{m=1}^M (f_n)_m^2 (\hatW_j')_m^2
    \leq   \max_m (\hatW'_j)_m^2 \sum_{m=1}^M (f_n)_m^2 =  \|\hatW'_j\|_\infty^2 \|f_n\|^2 .   \label{eqn:muf1}
  \end{align}
  Therefore, using Equation \eqref{eqn:fw_j}, we have
  \begin{align*}
    \E[(e_n^T F P_n^{(2,i)} \hatW')_j^2] & = p^{(2)} \sum_{m=1}^M (\tilde f_m^{(n,j)})^2                                 \\
                                         & = p^{(2)} \|\tilde f^{(n,j)}\|^2 \leq p^{(2)} \|\hatW'_j\|_\infty^2 \|f_n\|^2 \\
                                         & \leq p^{(2)} \frac{8 \mu (\hatW) \log M}{M} \|f_n\|^2 .
  \end{align*}
  We can sum over the $Nr$ coordinates of $ F P_n^{(2,i)} \hatW'$ to get
  \begin{align}
    \sum_{n=1}^N \left(\sum_{j=1}^r \E[(e_n^T F P_n^{(2,i)} \hatW')_j^2]\right)
     & \leq \sum_{n=1}^N \left(\sum_{j=1}^r p^{(2)} \frac{8 \mu (\hat W) \log M}{M} \|f_n\|^2\right) \nonumber \\
     & =  \sum_{n=1}^N \left(r p^{(2)} \frac{8 \mu (\hat W) \log M}{M} \|f_n\|^2\right)\nonumber               \\
     & =  r p^{(2)} \frac{8 \mu (\hat W) \log M}{M} \left(\sum_{n=1}^N\|f_n\|^2\right)\nonumber                \\
     & =   r p^{(2)} \frac{8 \mu (\hat W) \log M}{M} \|F\|_F^2 \label{eqn:G2_exp_general}
  \end{align}
  That is,
  \begin{align*}
    \E[\|\tilde G_A^{(2,i)}\|^2] & \leq  rp^{(2)} \frac{8 \mu (\hatW) \log M}{M} \|U\Sigma((I_M - \hatW \hatW^T)V)^T\|_F^2                  \\
                                 & \leq  rp^{(2)} \frac{8 \mu (\hatW) \log M}{M} \|U\|_F^2 \|\Sigma\|^2 \|((I_M - \hat W \hat W^T)V)^T\|^2  \\
                                 & \leq  rp^{(2)} \frac{8 \mu (\hatW) \log M}{M} r \frac{9}{4} \sigma_1(\X)^2 M  (\sin \theta (V, \hatW))^2 \\
                                 & =  18r^2p^{(2)}  \mu (\hatW) \log M  \sigma_1(\X)^2  (\sin \theta (V, \hatW))^2  .
  \end{align*}
  and
  \begin{align*}
    \E[\|\tilde G_Z^{(2,i)}\|^2] & \leq  rp^{(2)} \frac{8 \mu (\hatW) \log M}{M} \|Z\|_F^2                \\
                                 & \leq  rp^{(2)} \frac{8 \mu (\hatW) \log M}{M} MN \frac{3}{2}\sigma_Z^2 \\
                                 & =  12rp^{(2)}  \mu (\hatW) (\log M) N \sigma_Z^2.
  \end{align*}

  By Markov's Inequality,
  \begin{align*}
    \Pr[\|\tilde G_A^{(2,i)}\|_F^2 \geq 10 \E[\|\tilde G_A^{(2,i)} \|_F^2]] \leq \frac{1}{10},
  \end{align*}
  so with probability at least $\frac{9}{10} - M^{-2}$,
  \begin{align*}
    \|\tilde G_A^{(2,i)}\|_F^2 \leq  180r^2p^{(2)}  \mu (\hatW) \log M  \sigma_1(\X)^2  (\sin \theta (V, \hatW))^2  ,
  \end{align*}
  and similarly, with probability at least $\frac{9}{10} - M^{-2}$
  \begin{align*}
    \|\tilde G_Z^{(2,i)}\|_F^2 \leq  120rp^{(2)}  \mu (\hatW) (\log M) N \sigma_Z^2 .
  \end{align*}

  Therefore, with probability at least $\frac{4}{5} - 3M^{-2}$,

  \begin{align*}
    \|G_A^{(2,i)}\|_F & \leq  \|\tilde G_A^{(2,i)} \|_F  \max_{n \in [N]} \|(B_n^{(2,i)})^{-1}\|                                     \\
                      & \leq  \|\tilde G_A^{(2,i)} \|_F  \frac{2}{p^{(2)}}                                                           \\
                      & \leq  \left(6 \sqrt{5p^{(2)} \mu(\hatW)\log M } r \sigma_1(\X) \sin \theta(V,\hatW) \right)\frac{2}{p^{(2)}} \\
                      & =   \frac{12 \sqrt{5 \mu(\hatW)\log M } r \sigma_1(\X) \sin \theta(V,\hatW) }{\sqrt{p^{(2)}}}
  \end{align*}
  and

  \begin{align*}
    \|G_Z^{(2,i)}\|_F
     & \leq  \|\tilde G_Z^{(2,i)} \|_F^2  \frac{2}{p^{(2)}}                                     \\
     & \leq  \left(2 \sqrt{30r p^{(2)} \mu(\hat W) (\log M) N} \sigma_Z\right)\frac{2}{p^{(2)}} \\
     & =  \frac{4\sigma_Z \sqrt{30r  \mu(\hat W) (\log M) N} }{\sqrt{p^{(2)}}}.
  \end{align*}

  This implies the desired result.

\end{proof}

\begin{lemma}
  \label{lemma:minorize}
  Suppose $\sin \theta(U,X) \leq \epsilon$ and let
  \begin{align}
    \sigma_*(U;k^{(1)}) :=  \min_{S \subset [N], |S| = k^{(1)}} \sigma_r(\Q_S(U))
  \end{align}
  Then $\Q_S(X)$ has minimum singular value at least
  $\sigmaU -\epsilon$.
  In particular, if $\epsilon \leq \sigmaU/2$, then $\sigma_r(\Q_S(X)) \geq \sigmaU/2$.
\end{lemma}

\begin{proof}[Proof of Lemma \ref{lemma:minorize}]
  First, note that $\sqrt{1-\epsilon^2}  \leq \cos \theta (U,X) = \sigma_r(X^TU)$, so that $X^TU$ is invertible, with $\sigma_r(X^TU) \leq 1$, which implies $1 \leq \|(X^TU)^{-1}\|$. Choose any $v \in \R^r$ with $\|v\|=1$. Our goal is to show that $\Q_S(X) v \geq \sigmaU - \epsilon$. Let $\tilde P_S = \sum_{i = 1}^{k^{(1)}} e_i e_{s_i}^\top \in \R^{k^{(1)} \times N}$. Then we have
  \begin{align}
    \Q_S(X)v                & = \tilde P_S Xv                                                                             \\
                            & = \tilde P_S(X (X^TU)(X^TU)^{-1}v) = \tilde P_S(U(X^TU)^{-1}v - (U - X (X^TU))(X^TU)^{-1}v) \\
    \implies   \|\Q_S(X)v\| & \geq \|\tilde P_S(U(X^TU)^{-1}v\| - \|\tilde P(U - X X^TU)(X^TU)^{-1}v)\|                   \\
                            & \geq \sigma_r(\tilde P_SU) \|(X^TU)^{-1}v\| - \|\tilde P(U - X X^TU)(X^TU)^{-1}\|           \\
                            & \geq \sigma_r(\Q_S(U))\|(X^TU)^{-1}v\|  - \|\tilde P(I - X X^T) U\| \|(X^TU)^{-1}v\|        \\
                            & \geq \sigma_r(\Q_S(U)) - \epsilon.
  \end{align}

\end{proof}

\begin{lemma}
  \label{lemma:worst_omega}
  There exists  a set $(\Omega^{(1)})_m$ of size $k^{(1)}$ such that $\sigma_r(P_m^{(1)} X^\prev) \leq \sqrt{p^{(1)}}$.
\end{lemma}
\begin{proof}
  Since $\lambda_{\min}$ is concave,  we have $\E[\lambda_{\min}[(X^\prev)^T P_m^{(1)} X^\prev] \leq \lambda_{\min} (p^{(1)} I_r) = p^{(1)}$ where the expectation is uniform over all subsets of size $k^{(1)}$. Therefore, there exists at least one $\Omega^{(1)}$ such that $\lambda_{\min} ((X^\prev)^T P_m^{(1)} X^\prev)  \leq p^{(1)}$, from which the conclusion follows.
\end{proof}

\section{Noisy Subspace Iteration proofs}
\label{section:deferred}
Here we have the proofs of the noisy subspace iteration Lemmas, which are very similar to the ones in \cite{hardt}, but include for completeness.

\begin{proof}[Proof of Lemma \ref{lemma:noisy_subspace_W}]
  First, we verify that $W$ has rank $r$; this happens iff $\tilde W$ has rank $r$, and
  \[
    \sigma_r(\tilde W) \geq \sigma_r(V^T \tilde W) = \sigma_r(V^T(A^T X + G^{(1)})) = \sigma_r(\Sigma U^T X + V^T G^{(1)})  \geq \sigma_r(A) \sigma_r(U^T X) - \|V^T G^{(1)}\|
  \]
  By the assumptions of the lemma, this is positive. By Proposition 3.2 of \cite{zhu2013angles}, $\tan\theta (V, W) = \|V\p^T W (V^T W)^{-1}\|$, so
  \begin{align*}
    \tan \theta (W, V) & = \|V\p^T W (V^T W)^{-1}\| = \|V\p^T \tilde W(V^T \tilde W)^{-1}\|        \\
                       & = \|V\p^T \tilde W(V^T (A^T X + G^{(1)}) )^{-1}\|                         \\
                       & = \|V\p^T \tilde W ( \Sigma U^T X  + V^TG^{(1)})^{-1}\|                   \\
                       & = \|V\p^T \tilde  W (U^TX)^{-1}(\Sigma  + V^TG^{(1)}(U^T X)^{-1})^{-1}\|.
  \end{align*}
  Letting $S =  \Sigma  + V^TG^{(1)}(U^T X)^{-1}$,
  \begin{align*}
    \tan \theta (W, V) & = \|V\p^T\tilde W (U^TX)^{-1}(S)^{-1}\|                  \\
                       & \leq \frac{\|V\p^T \tilde W (U^TX)^{-1}\|}{\sigma_r(S)}.
  \end{align*}
  To bound the numerator,
  \begin{align*}
    \|V\p^T\tilde W (U^TX)^{-1}\| & = \|V\p^T (A^T X + G^{(1)}) (U^T X)^{-1}\|                 \\
                                  & = \|V\p^T (V\p \Sigma\p U\p^T  X + G^{(1)}) (U^T X)^{-1}\| \\
                                  & = \|V\p^TG^{(1)} (U^T X)^{-1}\|                            \\
                                  & \leq  2\| V\p^T G^{(1)}\|.
  \end{align*}
  The lower bound on $\sigma_r(S)$ is
  \begin{align*}
    \sigma_r( \Sigma  + V^TG^{(1)}(U^T X)^{-1}) & \geq \sigma_r(\Sigma) - \|V^T G^{(1)}(U^T X)^{-1}\| \\
                                                & \geq \sigma_r(A) - 2\|V^T G^{(1)}\|.
  \end{align*}
  Putting these together,
  \begin{align}
    \tan \theta (W, V) \leq \frac{  2\| V\p^T G^{(1)}\|}{\sigma_r(A) -  2 \|V^T G^{(1)}\|}    \label{eqn:noisy_subspace_W_gen}.
  \end{align}
\end{proof}

\begin{proof}[Proof of Lemma \ref{lemma:noisy_subspace}]
  First we verify that $X$ has rank $r$:  $X$ has rank $r$ if $XR$ has rank $r$,
  \begin{align*}
    \sigma_r(XR) & \geq \sigma_r(U^T XR)                                                                 \\
                 & = \sigma_r(U^TA(A^T X^\prev  +  G^{(1)}) + U^TG^{(2)}R)                               \\
                 & \geq \sigma_r(A(A^T X^\prev  +  G^{(1)})) - \| U^TG^{(2)}R\|                          \\
                 & = \sigma_r(U \Sigma (\Sigma U^T X^\prev  + V^T G^{(1)})) - \| U^TG^{(2)}R\|           \\
                 & \geq \sigma_r(\Sigma)\sigma_r(\Sigma U^T X^\prev  + V^T G^{(1)}) - \| U^TG^{(2)}R\| .
  \end{align*}

  By the assumptions of the lemma, this is positive. By Proposition 3.2 of \cite{zhu2013angles}, $\tan(U, X) = \|U\p^T X (V^T X)^{-1}\|$, so
  \begin{align*}
    \tan \theta(U, X) & =   \|U_{\perp}^T X(U^T X)^{-1} \|                                                                                       \\
                      & = \|U\p^T XR(U^T XR)^{-1}\| = \|U\p^T XR(U^T (AA^T X^\prev + A G^{(1)} + G^{(2)}R))^{-1}\|                               \\
                      & = \|U\p^T XR(\Sigma^2 U^T X^\prev + \Sigma V^T G^{(1)} + U^TG^{(2)}R)^{-1}\|                                             \\
                      & = \|U\p^T XR(U^T X^\prev)^{-1}(\Sigma^2 + \Sigma V^T G^{(1)} (U^T X^\prev)^{-1}+ U^TG^{(2)}R(U^T X^{\prev})^{-1})^{-1}\|
  \end{align*}
  Letting $S = \Sigma^2 + \Sigma V^T G^{(1)} (U^T X^\prev)^{-1}+ U^TG^{(2)}R(U^T X^\prev)^{-1}$,
  \begin{align*}
    \tan \theta(U, X) & =   \|U\p^T XR(U^T X^\prev)^{-1}S^{-1} \|                  \\
                      & \leq   \|U\p^T XR(U^T X^\prev)^{-1}\| \cdot \|S^{-1} \|    \\
                      & \leq   \frac{\|U\p^T XR(U^T X^\prev)^{-1}\|}{\sigma_r(S)}.
  \end{align*}
  To bound the numerator,
  \begin{align*}
    \|U_{\perp}^T XR(U^T X^\prev)^{-1}\| & =   \|U_{\perp}^T (AA^T X^\prev + AG^{(1)} + G^{(2)}R)(U^T X^\prev)^{-1}\|                                                                           \\
                                         & =   \|U_{\perp}^T ((U \Sigma^2 U^T + U\p \Sigma\p^2 U\p^T) X^\prev + (U\Sigma V^T + U\p \Sigma\p V\p^T)G^{(1)} + G^{(2)}R)(U^T X^\prev)^{-1}\|       \\
                                         & =   \|(\Sigma\p^2 U\p^T X^\prev +  \Sigma\p V\p^TG^{(1)} + U\p^T G^{(2)}R)(U^T X^\prev)^{-1}\|                                                       \\
                                         & \leq   \|(\Sigma\p^2 U\p^T X^\prev(U^T X^\prev)^{-1}\| +  \|\Sigma\p V\p^TG^{(1)}\|\|(U^T X^\prev)^{-1}\| +\| U\p^T G^{(2)}R\|\|(U^T X^\prev)^{-1}\| \\
                                         & \leq   \|(\Sigma\p^2 U\p^T X^\prev(U^T X^\prev)^{-1}\| +  2\|\Sigma\p V\p^TG^{(1)}\| + 2\| U\p^T G^{(2)}R\|                                          \\
                                         & \leq   \|\Sigma\p^2\|\| U\p^T X^\prev(U^T X^\prev)^{-1}\| +  2\|\Sigma\p V\p^TG^{(1)}\| + 2\| U\p^T G^{(2)}R\|                                       \\
                                         & \leq   \sigma_{r+1}(A)^2 \tan \theta(U, X^\prev) +  2\|\Sigma\p V\p^TG^{(1)}\| + 2\| U\p^T G^{(2)}R\|                                                \\
                                         & \leq   \sigma_{r+1}(A)^2 \tan \theta(U, X^\prev) +  2 \sigma_{r+1}\| V\p^TG^{(1)}\| + 2\| U\p^T G^{(2)}R\|.
  \end{align*}

  The lower bound on $\sigma_r(S)$ is
  \begin{align*}
    \sigma_r(S) & \geq  \sigma_r(\Sigma (\Sigma + V^T G^{(1)}(U^TX^\prev)^{-1})) -2 \|U^T G^{(2)}R\| \\
                & \geq  \sigma_r(\Sigma) (\sigma_r(\Sigma)- 2\| V^T G^{(1)}\|) -2 \|U^T G^{(2)}R\|   \\
  \end{align*}

  Then
  \begin{align}
    \tan \theta (U, X) \leq \frac{\sigma_{r+1}(A)^2 \tan \theta(U, X^\prev) +  2 \sigma_{r+1}(A)\| V\p^TG^{(1)}\| + 2\| U\p^T G^{(2)}R\|}{ \sigma_r(\Sigma) (\sigma_r(\Sigma)- 2\| V^T G^{(1)}\|) -2 \|U^T G^{(2)}R\|}    \label{eqn:noisy_subspace_general}
  \end{align}
  If $A$ has rank $r$, $\sigma_{r+1}(A) = 0$, so we have
  \begin{align*}
    \tan \theta (U, X) \leq \frac{ 2\| U\p^T G^{(2)}R\|}{\sigma_r(\Sigma) (\sigma_r(\Sigma)- 2\| V^T G^{(1)}\|) -2 \|U^T G^{(2)}R\|}.
  \end{align*}
\end{proof}

\begin{proof}[Proof of Lemma \ref{lemma:G_expression}]
  We have, using Lemma 4.1 of \cite{hardt}
  \begin{align*}
    e_m^T \tilde W_0     & = e_m^T Y^T P_m^{(1)} X^{\prev}(B_m^{(1)})^{-1} \\
    e_n^T \tilde X^{(i)} & =  e_n^T Y P_n^{(2)} \hat W(B_n^{(2)})^{-1}
  \end{align*}
  We want to write $\tilde W_0 = A^T X^{\prev} + G^{(1)}_A + G^{(1)}_Z$.
  \begin{align*}
    e_m^T\tilde W_0 & =    e_m^T Y^T P_m^{(1)} X^{\prev} (B_m^{(1)})^{-1}                                                                                 \\
                    & =  e_m^T (A + Z)^T P_m^{(1)} X^{\prev} (B_m^{(1)})^{-1}                                                                             \\
                    & =  e_m^T A^T P_m^{(1)} X^{\prev} (B_m^{(1)})^{-1} +  \underbrace{e_m^T  Z^T P_m^{(1)} X^{\prev} (B_m^{(1)})^{-1}}_{e_m^T G^{(1)}_Z}
  \end{align*}
  by our definition for $G^{(1)}_Z$. Next, we have
  \begin{align*}
    e_m^T A^T P_m^{(1)} X^{\prev} (B_m^{(1)})^{-1} & = e_m^T A^T X^\prev - (e_m^T A^T X^\prev - e_m^T A^TP_m^{(1)} X^{\prev} (B_m^{(1)})^{-1})                                                \\
                                                   & = e_m^T A^T X^\prev - (e_m^T (V\Sigma U^T) X^\prev - e_m^T (V \Sigma U^T) P_m^{(1)} X^{\prev} (B_m^{(1)})^{-1} )                         \\
                                                   & = e_m^T A^T X^\prev -e_m^T V\Sigma (U^T X^\prev -  U^T P_m^{(1)} X^{\prev} (B_m^{(1)})^{-1})                                             \\
                                                   & = e_m^T A^T X^\prev -e_m^T V\Sigma (U^T X^\prev B_m^{(1)}-  U^T P_m^{(1)} X^{\prev}) (B_m^{(1)})^{-1}                                    \\
                                                   & = e_m^T A^T X^\prev -e_m^T V\Sigma (U^T X^\prev (X^\prev)^T P_m^{(1)} X^\prev-  U^T P_m^{(1)} X^{\prev}) (B_m^{(1)})^{-1}                \\
                                                   & = e_m^T A^T X^\prev -e_m^T V\Sigma (U^T (X^\prev (X^\prev)^T-I_N) P_m^{(1)} X^\prev) (B_m^{(1)})^{-1}                                    \\
                                                   & = e_m^T A^T X^\prev +\underbrace{e_m^T V\Sigma (((I_N- X^\prev (X^\prev)^T)U)^T P_m^{(1)} X^\prev) (B_m^{(1)})^{-1}}_{e_m^T G^{(1)}_{A}}
  \end{align*}
  By exactly the same calculation,
  \begin{align*}
    \tilde X^{(i)} =  A(\hat W) + G^{(2,i)}_A + G^{(2,i)}_Z,
  \end{align*}  and  combining with $\hat W = \tilde W R^{-1} = (\tilde W_0 + G^{(1)}_H)R^{-1}  = (A^T X^\prev + G^{(1)})R^{-1}$,
  \begin{align*}
    \tilde X^{(i)} =  A(A^T X^\prev + G^{(1)})R^{-1} + G^{(2,i)}_A + G^{(2,i)}_Z.
  \end{align*}

\end{proof}
\subsection{Smooth QR}
\label{subsection:smoothQR}
This section on SmoothQR has  exactly the same material as from \cite{hardt}, but we restate it for our setting and notation.

\begin{subroutine}[H]
  \begin{algorithmic}
    \Function{SmoothQR}{$\tilde W_0, \epsilon, \mu $}
    \State{$\hat W \gets GS(\tilde W_0),  \sigma \gets \epsilon \|\tilde W\|/M$}
    \While{$\mu(\hat W) > \mu$ and $\sigma \leq \|\tilde W\|$}
    \State{$\tilde W \gets \tilde W_0 + G_H$ where $G_H \sim \N(0, \sigma^2/M)$}
    \State{$\hat W  \gets GS(\tilde W)$}
    \State{$\sigma \gets 2 \sigma$}
    \EndWhile\\
    \Return{$(\hat W, \tilde W, G_H)$}
    \EndFunction
  \end{algorithmic}
  \caption{\SmoothQR \citep{hardt}: Smooth Orthonormalization}
  \label{alg:SmoothQR}
\end{subroutine}

\begin{definition}[$\rho$-coherence \citep{hardt}]
  Given a matrix $G \in \R^{M \times r}$ where $M \geq r$, we let $\rho(G) := \frac{M}{r}\max_{m \in [M]}\|e_n^T G\|^2$.
\end{definition}

\begin{lemma}[Lemma 5.4 from \citep{hardt}]
  \label{lemma:smoothQR}
  Let $\tau > 0$ and assume $k = o(M/\log M)$.
  and $\mu(V) r\leq M$.
  Then there is an absolute constant $C' > 0$ such that the following holds. Let $G \in \R^{M \times r}$ and
  let $\nu \geq \|G\|$.
  Assume that
  \[
    \mu \geq \frac{C'}{\tau^2} \left(\mu(V) + \frac{ \rho(G) }{\nu^2} + \log M\right)
  \]
  Then, if $\epsilon' \leq \tau \nu$ satisfies $\log(M /\epsilon) \leq M$ and $\mu \leq M$, we have with probability $1-O(M^{-4})$, the algorithm SmoothQR($A^T X +  G, \epsilon', \mu$) terminates in $O(\log (M/\epsilon'))$ steps and outputs $( W, H)$ such that $\mu(W) \leq \mu$ and where $H$ satisfies $\|H\| \leq \tau\nu$.
\end{lemma}

\section{Concentration Inequalities and Random Matrix Theory}
\label{section:concentration}
\begin{theorem}[Chernoff (Theorems 10.1, 10.7 from \cite{doerr2018probabilistic})]
  \label{thm:scalar_chernoff}
  Let $X_1,\ldots,X_M$  be independent random variables taking values in $[0,1]$. Let $X = \sum_{m=1}^M X_M$ and $\mu = \E[X]$. Then for $\delta > 1$,
  \[
    \Pr[X \geq (1 + \delta) \mu] \leq \exp(- \mu \delta / 3).
  \]
  For all $\lambda \geq 0$,
  \[
    \Pr[X \geq \mu + \lambda ] \leq \exp\left(\frac{-2 \lambda^2}{M}\right).
  \]
\end{theorem}
\begin{theorem}[Corollary 5.35, \cite{vershynin2010introduction}]
  \label{thm:singular_vals}
  Let $A$ be an $N \times n$ matrix whose entries are independent standard normal variables. Then for every $t \geq 0$, with probability at least $1- 2 \exp(-t^2/2)$, one has
  \[
    \sqrt{N} - \sqrt{n} - t \leq \sigma_{\min}(A) \leq \sigma_{\max}(A) \leq \sqrt{N} + \sqrt{n} + t.
  \]
\end{theorem}

\begin{theorem}[Matrix Chernoff (Corollary 5.2 from \cite{tropp2012user})  ]
  \label{thm:chernoff}
  Consider a finite sequence $\{X_k\}$ of independent, random, self-adjoint $d \times d$ matrices that satisfy
  \[
    X_k \succeq 0 \text{ and } \lambda_{\max}(X_k) \leq R \text{ almost surely}.
  \]
  Let $\mu_{\min} = \lambda_{\min} \left(\sum_k \E[ X_k]\right)$. Then
  \[
    \Pr [ \lambda_{\min} \left(\sum_k X_k \right) \leq t \mu_{\min}] \leq d  \exp\left(-\frac{(1-t)^2 \mu_{\min} }{ 2R}\right).
  \]
\end{theorem}
\begin{theorem}[Matrix Hoeffding, Theorem 1.3 from \cite{tropp2012user}]
  \label{thm:hoeffding}
  Consider a finite sequence $\{X_k\}$ of independent, random, self-adjoint matrices with dimension $d$, and let $\{A_k\}$ be a sequence of fixed self-adjoint matrices. Assume that each random matrix satisfies
  \[
    \E[X_k] = 0 \text{ and }  X_k^2 \preceq A_k^2 \text{ almost surely}.
  \]
  Then, for all $\tau \geq 0$,
  \[
    \Pr[\lambda_{\max} \left(\sum_k X_k \right) \geq \tau ] \leq d \exp\{-\tau^2/8\sigma^2\},
  \]
  where $\sigma^2 := \| \sum_k A_k^2\|$.
\end{theorem}
\begin{theorem}[Wedin's Theorem]
  \label{thm:wedin}
  Let $A_0, Z \in \R^{N \times M}$ and let $A_1 = A_0 + Z$. Assume for some $k \geq 1$, $\sigma_k(A) \geq \sigma_{k+1}(A) + \|Z\|$. For $a \in \{0,1\}$ let $P_a$ denote the projector onto the space spanned by the first $k$ right singular vectors of $A_a$. Then
  \[
    \|(I_N - P_0) P_1 \| \leq \frac{ \|Z\|}{\sigma_k(A_0) - \sigma_{k+1}(A_0) - \|Z\|}.
  \]
\end{theorem}

\section{Scaled PCA Estimator}
Here $\one_N \in \R^{N \times N}$ is the matrix with each entry equal to $1$, and $I_N \in \R^{N \times N}$ is the identity matrix.
\label{section:scaledpca}
\begin{subroutine}[H]
  \caption{\ScaledPCA }
  \label{alg:scaledPCA}
  \begin{algorithmic}[1]
    \Require{Partially observed $\P_{\Omega}(Y) \in \R^{N \times M}$; $k$, the number of entries per column, $N$ the number of rows of $\P_{\Omega}(Y)$}
    \Function{ScaledPCA}{$\P_{\Omega}(Y)$, $ k$, N}
    \State{$C \gets \P_{\Omega}(Y) \P_{\Omega}(Y)^T$}
    \LineComment{We denote by $\circ$ the Hadamard (elementwise) product}
    \State{$C_\scaled \gets \left(\frac{N^2}{k(k-1)} \one_N\right) \circ C +  \left(\left(\frac{N}{k} - \frac{N^2}{k(k-1)}\right) I_N\right)\circ C$} \label{alg:line_scale}
    \State{$\hat X \gets \QR(C_\scaled)$}
    \State{\Return{$\hat X$}}
    \EndFunction
  \end{algorithmic}
\end{subroutine}

\end{document}